\documentclass{article}
\pdfoutput=1

\usepackage{iclr2027_conference,times}

\usepackage[utf8]{inputenc}
\usepackage[T1]{fontenc}
\usepackage{hyperref}
\hypersetup{hidelinks,hypertexnames=false,bookmarksnumbered=true}
\usepackage{url}
\usepackage{placeins}
\usepackage{needspace}
\usepackage{booktabs}
\usepackage{longtable}
\usepackage{amsfonts}
\usepackage{nicefrac}
\usepackage{microtype}
\usepackage{graphicx}
\usepackage{amsmath}
\usepackage{amssymb}
\usepackage{amsthm}
\usepackage{mathtools}
\usepackage{algorithm}
\usepackage{algpseudocode}
\usepackage{multirow}
\usepackage{tikz}
\usetikzlibrary{arrows.meta,positioning,calc}

\graphicspath{{figures/}}

\AddToHook{env/table/begin}{\setlength{\belowcaptionskip}{3pt}}

\newtheorem{theorem}{Theorem}
\newtheorem{proposition}{Proposition}
\newtheorem{corollary}{Corollary}
\newtheorem{lemma}{Lemma}
\theoremstyle{definition}
\newtheorem{definition}{Definition}
\theoremstyle{remark}

\newcommand{\method}{\textsc{Locks}}%
\newcommand{\R}{\mathbb{R}}

\title{\method{}: Page-Local Compact Key Summaries for Efficient Long-Context Decoding}

\author{Junsung Hwang}

\iclrfinalcopy
\begin{document}
\maketitle
\lhead{\fontsize{8}{9}\selectfont\method{}: Page-Local Compact Key Summaries for Efficient Long-Context Decoding}

\begin{abstract}
Serving large language models at long context is bottlenecked by the
key-value (KV) cache, which is read at every decode step. We find that
attention keys are approximately low-rank within pages. A single
low-rank projection shared across pages can miss page-specific
directions; fitting a basis to each page better identifies the pages
receiving the most attention at comparable stored selector cost.
\method{} stores a rank-$r$ spectral summary per page, reconstructs its
within-page logits, and selects pages by log-sum-exp mass without reading
candidate keys or values. It stays within about a point of FullKV on
LongBench-v1, tracks the read-every-key exact-LSE oracle on RULER down to
the smallest budgets, and retains quality furthest under tight budgets on
AIME26 and MATH-500. At a $2048$-token budget it matches FullKV aggregate
quality beyond $100$K context while attending about $2\%$ of tokens.
Across ranks $2$--$8$, summaries use $4$--$10\%$ of full-KV bytes.
On~GH200 with GPU-resident KV, \method{} reduces complete decode-step time
by $1.8\times$ at $512$K context. With full KV offloaded to Grace
memory, it reaches $3.82$--$4.22\times$ the faster dense backend's
aggregate throughput at $64$K--$256$K by serving larger batches.
\end{abstract}

\section{Introduction}
\label{sec:intro}
Every decode step reads two things: the model's weights and the
sequence's key-value (KV) cache. At short context the step is
weight-bound; at long context and useful batch sizes, weights are shared
across requests but each request still reads its own cache. The cache
also limits the serving capacity. For Llama-$3.1$-$8$B~\citep{llama3_1model} with bf$16$ KV, a $128$K
context occupies roughly $16$ GiB per request, more than the model's
weights, and constrains both context length and concurrency.

Decode attention is sparse. A small fraction of the cache carries most
of a query's attention mass, and attending only that fraction often
recovers the full-cache output
\citep{zhang2023h2o,tang2024quest,ribar2024sparq}. The difficulty is
finding it without scoring every key, which would restore the traffic
sparsity removes. Selection therefore needs a small resident summary.

Sparse decoding succeeds only if its small budget is spent on the right
pages. Figure~\ref{fig:banner}d makes the stakes concrete on RULER-32K:
at $b{=}64$, \method{} scores $100\%$ on NIAH-S3 and $75.2\%$ on VT.
In the separate exclusion test, removing its selected pages while
retaining the rest of the cache yields $0\%$ and $10.4\%$, respectively.
The challenge is to identify useful pages without first reading every
candidate key.

That decision is limited by the resident page summary
(Fig.~\ref{fig:banner}a). Quest's coordinate-wise extrema give an
optimistic bound on a page's largest token logit; the extrema can come
from different tokens, so the bound need not describe any token in the
page. \method{} instead fits a compact basis to each page, reconstructs
its token logits, and uses their log-sum-exp to estimate the page's
attention mass. Exact page-LSE selection minimizes worst-case
attention-output error among deterministic value-blind rules, but
computing it reads every key (\S\ref{sec:setup}). Under a
common GQA-shared budget, Fig.~\ref{fig:banner}b tests the consequence
of these summaries in the attention output computed from original KV;
panel (c) tests whether fitting the basis helps beyond merely storing
the same number of arbitrary directions. The figure thus motivates
\emph{what the selector must preserve} and why a page-fitted spectral
summary is a plausible way to preserve it.

\begin{figure}[t]\centering
\includegraphics[width=\linewidth]{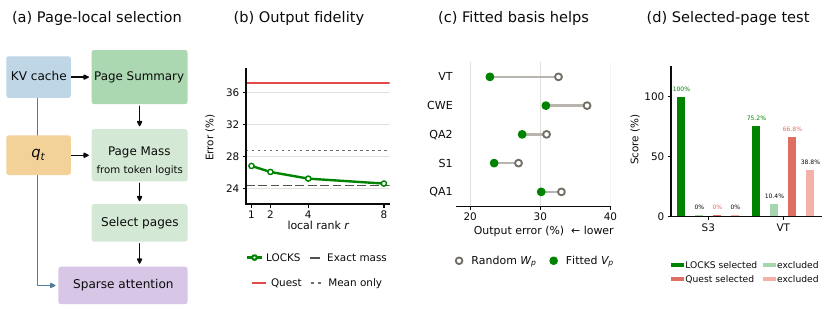}
\caption{\textbf{(a)}~\method{} scores page mass from a resident page summary,
then attends with original KV from the selected pages.
\textbf{(b)}~Relative \(\ell_2\) error against the full-cache attention
output before the output projection (Llama-3.1-8B-Instruct, $B{=}16$,
eight pages per KV head).
\textbf{(c)}~A separate controlled rank-$4$ ablation replaces random directions
with each page's fitted basis at the same summary size and token budget. Each pair connects one sequence's output errors.
\textbf{(d)}~RULER-32K S3 and VT at $b{=}64$, with rank-$4$ \method{} summaries. Excluded scores measure
accuracy after removing each method's selected pages while retaining sink
and current pages.}
\label{fig:banner}
\end{figure}

Fitting useful directions leaves a separate question: can the fitted
representation be reused across pages, or must its scope be local?

Table~\ref{tab:scope-main} compares page-local summaries with shared
bases that either retain or replace each page centroid at near-matched
storage and $b{=}128$. The page-local fit retains more attention mass
and recovers more of the exact-LSE page set. Fixed shared projections
can also have blind directions (Prop.~\ref{prop:blind}).

Page-local state must also be deployable. Unlike trained gates
\citep{gao2024seerattention}, \method{} is a drop-in backend for
an existing paged-attention serving engine. It builds one small factorization
when each page completes, then scores and ranks pages in a fused,
CUDA-graph-compatible kernel (\S\ref{sec:method}, \S\ref{sec:eval}).

We adopt exact page-LSE selection as the deterministic value-blind minimax target
(App.~\ref{app:proofs-coverage}) and make three contributions:
\begin{itemize}
\item \textbf{Locality finding and matched-state evidence.} We show that
page-local state improves retained mass and page recall, and that these
advantages survive matched summary payload (\S\ref{sec:law}).
\item \textbf{Shared-projection impossibility and local construction.}
We prove blind directions for fixed shared low-rank projections
(Prop.~\ref{prop:blind}) and bound the unquantized local summary from
reconstruction error through coverage and output error
(Thm.~\ref{thm:retain}).
\item \textbf{Quality and deployability.} We show oracle-tracking
quality at aggressive budgets and long-context acceleration without
changing the model or the serving engine (\S\ref{sec:eval}).
\end{itemize}

\section{Preliminaries}
\label{sec:setup}

A decoder layer holds per-head caches
$\{(\mathbf k_i,\mathbf v_i)\}_{i=1}^{T}$. A decode step forms a query
$\mathbf q$ and outputs
$\mathbf o=\sum_i \alpha_i \mathbf v_i$ with
$\alpha=\mathrm{softmax}(\mathbf q^\top\mathbf k/\sqrt d)$. Paged
serving \citep{kwon2023vllm} stores the cache in pages of $B$ consecutive
tokens. We write $P_j$ for page $j$ and select at page granularity.
We use a fixed per-head token budget $b$, corresponding to
$k=\lceil b/B\rceil$ page slots. The sink page and the most recent page
always occupy two of those slots. This keeps the attended working set
bounded as context length increases.

\Needspace{5\baselineskip}
Coverage bounds sparse decoding's output error. For a page set $S$, let
$\gamma_S(\mathbf q)$ be the fraction of total attention mass that falls
in $S$, and let $\mathbf o_S$ be attention restricted to $S$ with the
softmax renormalized. Dropping the complement shifts the output by
exactly $\mathbf o-\mathbf o_S
=\varepsilon\,(\bar{\mathbf v}_{S^c}-\bar{\mathbf v}_S)$, with
$\varepsilon=1-\gamma_S(\mathbf q)$ and
$\bar{\mathbf v}_S,\bar{\mathbf v}_{S^c}$ the mass-weighted mean values
of the kept and dropped sets \citep{tzachristas2025topk,cobs2026}.
At the model's actual values, equal coverage can still yield different
output errors.

\Needspace{7\baselineskip}
Selection precedes the value fetch, so a deterministic value-blind rule
chooses pages from keys and query. For fixed keys and query, its worst-case
output error over $\|\mathbf v_i\|\le R$ is exactly
$2R(1-\gamma_S)$ (App.~\ref{app:proofs-coverage}). The minimax choice
therefore maximizes coverage. Define the exact page log-sum-exp (LSE)
score
\begin{equation}
s_j(\mathbf q)=\log\sum_{i\in P_j}
\exp\!\left(\frac{\mathbf q^{\top}\mathbf k_i}{\sqrt d}\right).
\label{eq:exact-page-lse}
\end{equation}
Because $e^{s_j(\mathbf q)}$ is the page's unnormalized attention mass,
the highest-scoring free pages, together with the always-kept pages,
achieve that minimax choice. Computing the exact scores requires reading
every cached key, defeating the traffic saving of sparse decoding.
\method{} therefore approximates this value-blind minimax ranking from a
compact resident page summary.

\section{Why Page-Local State Matters}
\label{sec:law}

Page selection changes substantially with the scope of the low-rank key
fit. To examine centering and basis fitting, we compare page-local
fits with shared fits that either keep or replace the page centroid.
For each completed page $j$, we compute
\[
\boldsymbol\mu_j=\frac1B\sum_{i\in P_j}\mathbf k_i,\qquad
\mathbf H_j=\sum_{i\in P_j}(\mathbf k_i-\boldsymbol\mu_j)
(\mathbf k_i-\boldsymbol\mu_j)^{\top}.
\]
\method{} takes the top $r$ eigenvectors of $\mathbf H_j$. The hybrid
controls retain $\boldsymbol\mu_j$ but fit one basis to pooled
$\mathbf H_j$ across a sequence or calibration set. Strict shared fits
use one centroid and basis at either scope. Table~\ref{tab:scope-main}
compares these choices; App.~\ref{app:reprscope} gives the fitting and
budget details.
\begin{table}[htbp]\centering
\setlength{\abovecaptionskip}{4pt}
\caption{Centroid and basis scope on RULER-32K (Llama-3.1-8B-Instruct,
$B{=}16$, $b{=}128$ attended tokens per KV head). All rows use
$488$--$490$ B of selector state per page and KV head. Mass and recall
use common FullKV states and the same page-LSE scorer.}
\label{tab:scope-main}
\footnotesize
\setlength{\tabcolsep}{4pt}
\begin{tabular}{@{}llrrrr@{}}
\toprule
Centroid & Basis & Rank & Mass (\%) & Recall (\%) & Score (\%) \\
\midrule
Page & Page (\method{} r4i4) & 4 & \textbf{85.8} & \textbf{90.5} & \textbf{80.1} \\
\midrule
Page & Sequence & 20 & 84.4 & 81.3 & 77.0 \\
Page & Global 4K & 20 & 84.6 & 81.9 & 76.3 \\
Page & Global 32K & 20 & 84.5 & 81.9 & 75.2 \\
\midrule
Sequence & Sequence & 28 & 81.2 & 60.2 & 47.6 \\
Global 4K & Global 4K & 28 & 81.3 & 58.0 & 63.1 \\
Global 32K & Global 32K & 28 & 82.2 & 63.3 & 56.9 \\
\bottomrule
\end{tabular}
\end{table}

With page centroids retained, shared bases recall $81.3$--$81.9\%$ of
exact-LSE pages; wholly shared fits recall $58.0$--$63.3\%$.
The page-local basis reaches $90.5\%$ at similar storage. This
observation motivates fitting both parts of the summary per page.
Proposition~\ref{prop:blind} shows a possible failure of a fixed shared
basis even when page centroids are known.
\begin{proposition}[A blind direction for a fixed shared basis]
\label{prop:blind}
Fix any $\mathbf W\in\R^{d\times\rho}$ with $1\le\rho<d$ and page size
$B\ge2$. There are pages $P$ and $P_c$ whose centered keys each have rank
at most one and whose centroids and projected keys
$\{\mathbf W^\top\mathbf k_i\}$ agree, but
$s_{P_c}(\mathbf q)-s_P(\mathbf q)\to\infty$ for some query $\mathbf q$
as $c\to\infty$. Any scorer using only these shared-projection summaries
assigns the pages the same score. An unquantized rank-one basis fitted to
each page represents both pages' keys, and hence their LSE, exactly
(proof in App.~\ref{app:proofs-blindness}).
\end{proposition}

\Needspace{6\baselineskip}
\section{\method{}: Local Compact Key Summaries}
\label{sec:method}
Exact-LSE selection requires reading every cached key (\S\ref{sec:setup}),
while a fixed shared low-rank projection can miss within-page variation
(Prop.~\ref{prop:blind}). \method{} therefore fits a separate basis to
each page and uses its resident summary to select pages at each decode
step under the fixed per-head token budget $b$ of \S\ref{sec:setup}.
It is training-free and applies to both
standard GQA decoders and recent models with trained sparse attention (App.~\ref{app:extstudies-qsa}).

\phantomsection
\label{sec:lowrank}
\textbf{Building the summary} We fit the basis to cached keys. For RoPE
models, these are post-RoPE keys, paired with the post-RoPE query when
scoring; Table~\ref{tab:appC-rope} compares pre- and post-RoPE page-local fits.
For page $j$, compute:
\[
\boldsymbol\mu_j=\frac1B\sum_{i\in P_j}\mathbf k_i,\quad
\mathbf D_j=\bigl[(\mathbf k_i-\boldsymbol\mu_j)^{\top}\bigr]_{i\in P_j}
\in\R^{B\times d},\quad
\mathbf C_j=\mathbf D_j\mathbf V_j\in\R^{B\times r}.
\]
Here $\mathbf V_j\in\R^{d\times r}$ contains the top $r$ right singular
vectors of $\mathbf D_j$, and row $\mathbf c_i^{\top}$ of $\mathbf C_j$
gives $\tilde{\mathbf k}_i=\boldsymbol\mu_j+\mathbf V_j\mathbf c_i$.
The squared page residual is
$\tau_{j,r}:=\lVert\mathbf D_j-\mathbf C_j\mathbf V_j^{\top}\rVert_F^2
=\sum_{\ell>r}\sigma_{j,\ell}^2$. Before quantization, reconstruction
is exact if $r\ge\operatorname{rank}(\mathbf D_j)$.

In labels such as r8i4, the first number is the rank and the second is
the basis bit width. The i4 variants store
$\mathbf V_j$ in int4 and $\mathbf C_j,\boldsymbol\mu_j$ in int8, requiring
$dr/2+Br+d$ bytes per page before scales (App.~\ref{app:h200-eff-config}).
Completed pages are summarized once, off the decode critical path.

\textbf{Scoring from the summary} Substitute the
reconstructed keys into the page LSE. The common centroid logit then
factors out:
\begin{equation}
\begin{aligned}
\hat s_j(\mathbf q)
&=\log\sum_{i\in P_j}\exp\!\left(
\frac{\mathbf q^{\top}(\boldsymbol\mu_j+\mathbf V_j\mathbf c_i)}{\sqrt d}
\right)\\
&=\frac{\mathbf q^{\top}\boldsymbol\mu_j}{\sqrt d}
+\log\sum_{i\in P_j}\exp\!\left(
\frac{(\mathbf V_j^{\top}\mathbf q)^{\top}\mathbf c_i}{\sqrt d}
\right),
\end{aligned}
\label{eq:score}
\end{equation}
where $\mathbf V_j^{\top}\mathbf q$ is computed once per page for each decode query. The fused ranking kernel
reads summaries, not candidate keys or values.

For each query $\mathbf q_g$ in a GQA group $G$, normalize its page scores to
obtain $\hat m_g(j)=e^{\hat s_j(\mathbf q_g)}/\sum_{j'}e^{\hat
s_{j'}(\mathbf q_g)}$. The share-average combine scores page $j$
by $\bar{\hat m}_j=\tfrac1G\sum_{g=1}^G\hat m_g(j)$. We always keep the
sink and most recent pages, then fill the other $k{-}2$ slots with the
highest-scoring remaining pages. All $G$ query heads use this same $k$-page set for
their shared KV head.

This averaging follows TokenSelect's head soft vote
\citep{tokenselect2024}. With exact shares, it maximizes group-mean
retained mass among shared $k$-page sets containing the two always-kept
pages (App.~\ref{app:proofs-combine}). Thm.~\ref{thm:retain}(ii) bounds
the loss from estimated shares, and App.~\ref{app:extstudies-heads}
reports worst-head measurements.

\begin{theorem}[Retention of the group-shared selection]
\label{thm:retain}
Fix a decode step with $J$ pages in one KV head, $J\ge k\ge2$, queries
$\{\mathbf q_g\}_{g=1}^G$, and values bounded by $\lVert\mathbf v_i\rVert\le R$.
Let $S$ contain the sink and most recent pages and the $k{-}2$ other pages
selected by $\bar{\hat m}_j$. For fully unquantized rank-$r$ summaries,
let $\eta:=\max_{g,j}|\hat s_j(\mathbf q_g)-s_j(\mathbf q_g)|$, with $\hat
s_j$ defined by Eq.~\eqref{eq:score}.
\par\noindent (i) The score error satisfies
$\eta\le\max_{g,j}\lVert\mathbf q_{g,j}^{\perp}\rVert\,
\sigma_{j,r+1}/\sqrt d$, where
$\mathbf q_{g,j}^{\perp}:=(\mathbf I-\mathbf V_j\mathbf V_j^{\top})
\mathbf q_g$ and $\sigma_{j,r+1}$ is page $j$'s first discarded singular
value (zero if none remains). In particular, $\eta=0$ if $r\ge\operatorname{rank}(\mathbf D_j)$
for every page.
\par\noindent (ii) The shared set satisfies
$\bar\gamma(S):=G^{-1}\sum_{g=1}^G\gamma_S(\mathbf q_g)
\ge e^{-4\eta}\bar\gamma^{\star}$, where $\bar\gamma^{\star}$ is the
maximum group-average coverage over shared $k$-page sets containing the
sink and most recent pages.
\par\noindent (iii) Consequently, the group-average output error obeys
\[
\frac1G\sum_{g\le G}\bigl\lVert\mathbf o_S(\mathbf q_g)-\mathbf o(\mathbf
q_g)\bigr\rVert \;\le\; 2R\bigl(1-e^{-4\eta}\,\bar\gamma^{\star}\bigr).
\]
\emph{Proof.} App.~\ref{app:proofs-retain}.
\end{theorem}

Part~(i) uses each page's top-$r$ SVD basis to bound score error from
discarded singular modes. That basis minimizes squared key reconstruction
error among rank-$r$ orthonormal bases (Cor.~\ref{cor:ey}).
Fig.~\ref{fig:quant-ablation} reports page recall for quantized summaries
against exact-LSE selection.

\paragraph{Implementation.}
\label{sec:impl}
\method{} integrates as an external attention backend in vLLM without
modifying the model or the engine's core. Because $r$,
$B$, and $b$ (hence $k$)
are fixed, the ranking kernel of \S\ref{sec:lowrank} has fixed shapes
too, so the
entire batched decode step, selection included, replays inside full CUDA
graphs. We analyze its serving cost, including the full-summary scan, in
\S\ref{sec:eval-eff}.

\FloatBarrier
\section{Experimental Evaluation}
\label{sec:eval}
Existing KV-cache optimization methods often focus on either retrieving evidence from long
inputs~\citep{tang2024quest,sun2025shadowkv} or sustaining long reasoning
traces~\citep{cai2025rkv,zhang2025lazyeviction}. We evaluate whether the
same \method{} page-local construction and selection rule serves both:
long-input retrieval and QA (\S\ref{sec:eval-acc}), and long-form
reasoning (\S\ref{sec:eval-reason}). Every quality figure sweeps the
fixed per-(layer, KV-head) attended-token budget $b$ of
\S\ref{sec:setup} over $b\in\{64,\dots,2048\}$ against FullKV, the
exact-LSE oracle (\S\ref{sec:setup}), the \method{} family (r8i4,
r4i4, r2i4), and a suite of
state-of-the-art selectors, evictors, and reasoning compressors named per
benchmark. Quality comparisons use identical records and a common FullKV
reference; the protocol, model roster, budget accounting, and per-arm
sampling settings are in App.~\ref{app:extres}. Separately, we measure
\method{} as a vLLM decode backend alongside FA3, FlashInfer, and a
GPU port of Quest (\S\ref{sec:eval-eff}).

\subsection{Retrieval and long-document QA}
\label{sec:eval-acc}
In retrieval and long-document QA, selection must find a few specific
pages from a long prefill. \method{} tracks exact-LSE selection
(\S\ref{sec:setup}) down to the smallest budgets and stays
within about a point of FullKV on
LongBench-v1~\citep{bai2024longbench} while attending to only
$0.6\%$--$18.6\%$ of its full KV cache (mean ${\sim}11$K tokens), whereas four state-of-the-art
selectors and evictors fall behind as the budget
shrinks (Fig.~\ref{fig:q-retr}), namely Quest~\citep{tang2024quest},
KVzip~\citep{kim2025kvzip}, ShadowKV~\citep{sun2025shadowkv}, and
RocketKV~\citep{behnam2025rocketkv}.
The gap is widest on RULER~\citep{hsieh2024ruler} at the smallest
budgets. Across the plotted sweep, this corresponds to attending to only
$0.4\%$--$12.5\%$ of the full RULER-16K cache and
$0.2\%$--$6.25\%$ of the full RULER-32K cache.
Fig.~\ref{fig:banner}d shows that excluding \method{}'s selected pages
sharply lowers S3 and VT accuracy at $32$K, even with the rest of the
cache attended. Exact-LSE also has limits as a task-accuracy reference:
RocketKV can surpass the oracle on NIAH-MK3, while \method{} and the
oracle fail on word-frequency aggregation at $b{=}64$
(App.~\ref{app:extres}).

\begin{figure}[tbp]\centering
\setlength{\abovecaptionskip}{6pt}
\includegraphics[width=\linewidth]{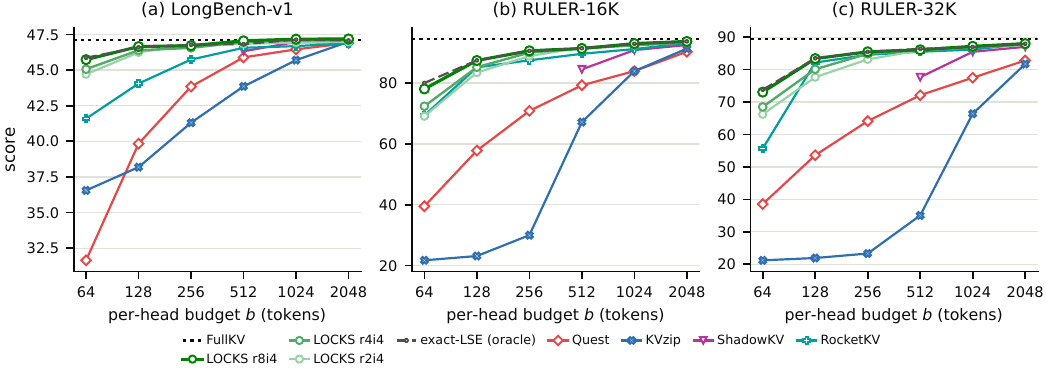}
\caption{Accuracy as a function of the attended-token budget
(Llama-3.1-8B, greedy decoding). Panels show (a) LongBench-v1
(14-subset mean, average context ${\sim}11$K tokens), (b) RULER-16K,
and (c) RULER-32K (13-task means).
The $x$-axis shows a common nominal budget. Actual KV reads and cache
residency differ by method (App.~\ref{app:extres-retrieval-baselines}).
ShadowKV begins at $b{=}512$ because its
$416$ always-attended tokens exceed the smaller budgets.}
\label{fig:q-retr}
\end{figure}
\FloatBarrier

We extend the comparison to $100$K+ contexts on
InfiniteBench~\citep{infinitebench2024}. With GLM-4-9B-Chat-1M~\citep{glm4},
the same ordering holds: \method{} tracks the exact-LSE oracle and has the highest measured
mean among deployable selectors at the tight $b{=}512$ and generous $b{=}2048$
budgets (Table~\ref{tab:infinitebench-main}). On a $100$K cache, these
correspond to attending to about $0.5\%$ and $2\%$ of tokens, respectively,
and less on longer records.

\begin{table}[!htbp]\centering
\setlength{\abovecaptionskip}{5pt}
\setlength{\belowcaptionskip}{6pt}
\caption{InfiniteBench scores at $100$K+ context (GLM-4-9B-Chat-1M).
Retr and LongQA average retrieval and long-document QA tasks; Avg covers
all $11$ tasks. Entries are means $\pm$ $95\%$ bootstrap CI half-widths.
Bold marks the leading deployable point estimate; the leading intervals
overlap. Per-task results are in Table~\ref{tab:appB-infinitebench}.}
\label{tab:infinitebench-main}
\providecommand{\cipm}[2]{#1\,{\scriptsize$\pm$#2}}
\setlength{\tabcolsep}{4pt}
\begin{tabular}{l rrr rrr}
\toprule
& \multicolumn{3}{c}{budget $b{=}512$} & \multicolumn{3}{c}{budget $b{=}2048$}\\
\cmidrule(lr){2-4}\cmidrule(lr){5-7}
Method & Retr & LongQA & Avg & Retr & LongQA & Avg \\
\midrule
FullKV             & \cipm{76.0}{4.1} & \cipm{35.3}{3.4} & \cipm{43.0}{2.1} & \cipm{76.0}{4.1} & \cipm{35.3}{3.4} & \cipm{43.0}{2.1} \\
Oracle (exact-LSE) & \cipm{74.7}{3.9} & \cipm{35.1}{3.7} & \cipm{42.3}{3.0} & \cipm{76.7}{4.2} & \cipm{36.2}{3.7} & \cipm{43.9}{2.8} \\
\midrule
\textbf{\method{} (r8i4)} & \cipm{\textbf{71.3}}{3.2} & \cipm{\textbf{33.6}}{3.7} & \cipm{\textbf{41.1}}{2.4} & \cipm{\textbf{76.0}}{4.1} & \cipm{\textbf{36.6}}{3.8} & \cipm{\textbf{43.6}}{2.6} \\
RocketKV           & \cipm{64.0}{5.2} & \cipm{33.6}{4.1} & \cipm{38.5}{2.7} & \cipm{75.3}{4.8} & \cipm{35.0}{4.2} & \cipm{42.3}{2.4} \\
ShadowKV           & \cipm{66.0}{1.3} & \cipm{29.1}{3.4} & \cipm{36.0}{2.1} & \cipm{71.3}{3.2} & \cipm{33.1}{3.8} & \cipm{40.3}{2.3} \\
Quest              & \cipm{62.7}{3.0} & \cipm{26.7}{3.5} & \cipm{34.7}{2.7} & \cipm{58.0}{4.0} & \cipm{31.0}{3.8} & \cipm{35.7}{2.6} \\
\bottomrule
\end{tabular}
\end{table}

\FloatBarrier

\subsection{Long-form reasoning}
\label{sec:eval-reason}
Long-form reasoning grows the KV cache over thousands of generated tokens
while the relevant past tokens can change. We evaluate this regime on
AIME26 (2026 I+II, $30$ problems) and
MATH-500~\citep{hendrycks2021math,lightman2024verify} with
Qwen3-4B~\citep{qwen3} in thinking mode. At the same attended-token
budgets, we compare \method{} with FullKV, the exact-LSE oracle, and four
baselines: Quest~\citep{tang2024quest}, R-KV~\citep{cai2025rkv},
TriAttention~\citep{mao2026triattention}, and
LazyEviction~\citep{zhang2025lazyeviction}.

Figure~\ref{fig:q-reason} shows a task-dependent budget floor. On AIME26,
\method{} tracks the exact-LSE oracle and reaches FullKV-level accuracy
by $b{=}512$, while all four baselines remain far below it. Quest still
scores $0\%$ there despite being able to reselect past KV at each step,
so irreversible eviction alone does not explain the gap. On MATH-500,
r8i4 matches FullKV at $b{=}1024$, and the baselines trail across the
sweep. Even at $b{=}64$, all three \method{} ranks score above $50\%$
on MATH-500, while no baseline reaches $12\%$. At this budget, even
the exact-LSE oracle scores $0\%$ on AIME26 but $56\%$ on MATH-500:
AIME26's collapse cannot be attributed to summary approximation alone.
At $b{=}2048$, the strongest AIME26 baseline still trails FullKV by
over $14$ points, whereas the MATH-500 baselines have largely recovered.
The r4i4 and r2i4
variants often match or exceed r8i4 on both
tasks. Full results are in
Tables~\ref{tab:appB-aime26} and~\ref{tab:appB-math500}.

\begin{figure}[!htbp]\centering
\includegraphics[width=\linewidth]{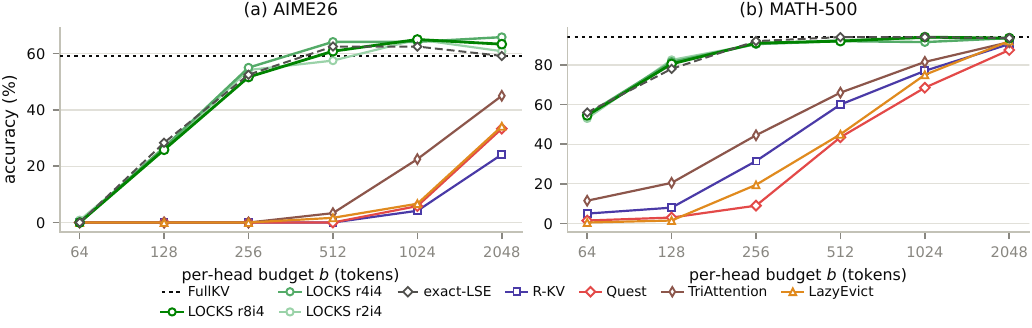}
\caption{Reasoning accuracy across attended-token budgets with
Qwen3-4B~\citep{qwen3}
(thinking on, temperature $0.6$, top-$p$ $0.95$, top-$k$ $20$,
avg@$4$). Panels show (a) AIME26 ($38{,}912$-token generation cap) and
(b) MATH-500 ($16{,}384$-token generation cap). Further details
are in App.~\ref{app:extres}.}
\label{fig:q-reason}
\end{figure}
\FloatBarrier

\Needspace{6\baselineskip}
\subsection{Efficiency}
\label{sec:eval-eff}
We evaluate \method{} on GH200 hardware after long prefills, during long
generation, and with batched requests. We time complete vLLM
decode steps, including page selection, against dense FA3 and
FlashInfer and an optimized Quest port~\citep{tang2024quest}.
Table~\ref{tab:appB-eff-full} also reports latencies for Quest-U, a
port adapted from the upstream implementation. The other sparse
quality baselines were not integrated into this timing setup.

With GPU-resident KV, \method{} and Quest have comparable first-token
and decode latencies on GLM-4-9B-Chat-1M, as shown in
Fig.~\ref{fig:eff-glm-dense}(a--c). \method{}'s TTFT includes summary
construction yet exceeds FA3's by no more than $6.4\%$ across
$16$K--$512$K prefills. At $512$K, it decodes in $10.1$\,ms per step
versus FlashInfer's $18.0$\,ms. Panel (d) compares quality and decode
latency on Qwen3-4B: at $b{=}2048$, \method{} scores $65.8\%$ on AIME26
versus Quest's $33.3\%$, with mean step times of $4.49$ and $4.58$\,ms,
respectively.

\begin{figure}[!htbp]\centering
\includegraphics[width=\linewidth]{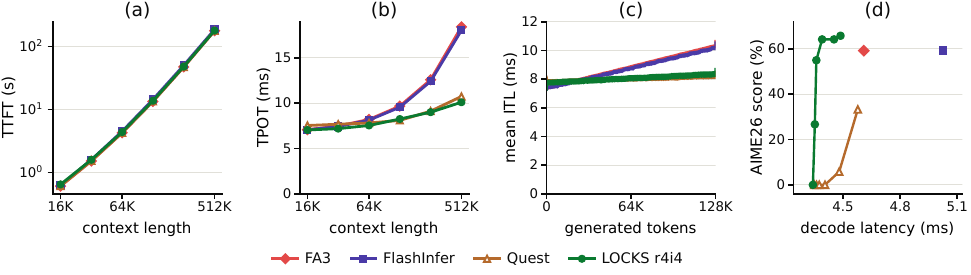}
\caption{GH200 decode efficiency (batch $1$, resident bf16 KV).
For GLM-4-9B-Chat-1M~\citep{glm4}, (a,b) show mean TTFT over three
prefills and mean decode time over $99$ steps at $b{=}2048$; (c) shows
$10$-token mean decode time during a $128$K continuation after a $32$K
prefill. (d) pairs Qwen3-4B~\citep{qwen3} AIME26 avg@$4$ scores (FullKV
for FA3 and FlashInfer) with mean decode time after a $100$-token input
and $38{,}912$ generated tokens.
Markers connect budgets $64$--$2048$ in ascending order.}
\label{fig:eff-glm-dense}
\end{figure}
\FloatBarrier

Serving systems batch requests to increase throughput~\citep{yu2022orca,kwon2023vllm}.
With its original KV offloaded to Grace memory, \method{} uses GPU
memory for a hot cache and larger batches. At $64$K, $128$K, and $256$K,
Table~\ref{tab:gh200-batched} reports $1{,}887.7$, $1{,}067.5$, and
$530.7$ tokens/s, or $3.82$--$4.22\times$ the faster dense backend.
Its measured batches ($58$, $29$, $14$ versus $12$, $6$, $3$) yield this
aggregate gain despite longer mean decode steps.

\begin{table}[H]
\centering\small
\setlength{\abovecaptionskip}{5pt}
\setlength{\belowcaptionskip}{6pt}
\caption{GH200 batched decode (GLM-4-9B-Chat-1M, TP$=1$, bf16).
Dense KV stays in GPU HBM; \method{} r4i4 ($b{=}2048$) uses Grace memory over coherent NVLink-C2C~\citep{syed2025kvsharing}
with a GPU hot cache. Hot KV is allocated GPU cache capacity. The
``Hot pages'' column gives the share of selected page accesses served
from that cache over $192$ decode steps. Throughput is aggregate generated tokens per second at the shown batch; Speedup
is its ratio to the faster dense row at the same context.}
\label{tab:gh200-batched}
\setlength{\tabcolsep}{3pt}
\begin{tabular*}{\linewidth}{@{\extracolsep{\fill}}l l r r r r r r@{}}
\toprule
Context & Method & Batch & Hot KV & Hot pages & Mean ITL & Throughput & Speedup \\
 & & & (GiB) & (\%) & (ms) & (tok/s) & \\
\midrule
\multirow{3}{*}{$64$K}
 & FA3 & $11$ & N/A & N/A & $22.5$ & $489.0$ & $0.99\times$ \\
 & FlashInfer & $12$ & N/A & N/A & $24.3$ & $493.9$ & $1.00\times$ \\
 & \method{} r4i4 & $58$ & $36.3$ & $90.3$ & $30.7$ & \textbf{1,887.7} & $\mathbf{3.82\times}$ \\
\midrule
\multirow{3}{*}{$128$K}
 & FA3 & $6$ & N/A & N/A & $24.2$ & $247.5$ & $0.98\times$ \\
 & FlashInfer & $6$ & N/A & N/A & $23.7$ & $253.1$ & $1.00\times$ \\
 & \method{} r4i4 & $29$ & $31.7$ & $91.0$ & $27.2$ & \textbf{1,067.5} & $\mathbf{4.22\times}$ \\
\midrule
\multirow{3}{*}{$256$K}
 & FA3 & $3$ & N/A & N/A & $24.1$ & $124.3$ & $0.98\times$ \\
 & FlashInfer & $3$ & N/A & N/A & $23.6$ & $127.0$ & $1.00\times$ \\
 & \method{} r4i4 & $14$ & $21.9$ & $89.2$ & $26.4$ & \textbf{530.7} & $\mathbf{4.18\times}$ \\
\bottomrule
\end{tabular*}
\end{table}

\clearpage

\subsection{KV-cache quantization under extreme sparsity}
\label{sec:eval-kivi}
Selection reads a few pages per step; KV quantization can shrink their
storage but may amplify error under sparse attention. We test
KIVI quantization~\citep{liu2024kivi} with two summary inputs:
\textsc{post} builds from bf16 keys, while \textsc{pre} builds from
quantized keys; both attend to quantized KV
(Table~\ref{tab:kivi-main}; full grid in
Table~\ref{tab:appB-kivi}). Four-bit KV remains compatible with rank-$8$
\method{} across every tested budget, while two-bit KV consistently
costs quality, especially at tight budgets. 
\begin{table}[H]\centering
\caption{RULER-32K average ($13$ tasks, Llama-3.1-8B) under $4$- and
$2$-bit KIVI quantization at three budgets. \textsc{post}
builds the summary from bf16 keys; \textsc{pre} builds it from quantized
keys. Both attend to quantized KV. The bf16 \method{} control matches
budget and rank. Full grid: Table~\ref{tab:appB-kivi}.}
\label{tab:kivi-main}
\scriptsize
\setlength{\tabcolsep}{2.6pt}
\begin{minipage}[t]{0.32\linewidth}\centering
$r{=}8$\\[2.5pt]
\begin{tabular}{r r rr rr}
\toprule
& & \multicolumn{2}{c}{4-bit} & \multicolumn{2}{c}{2-bit}\\
\cmidrule(lr){3-4}\cmidrule(lr){5-6}
$b$ & bf16 & \textsc{post} & \textsc{pre} & \textsc{post} & \textsc{pre}\\
\midrule
$64$   & $73.1$ & $72.8$ & $73.0$ & $68.0$ & $67.6$\\
$256$  & $85.5$ & $85.2$ & $85.2$ & $81.9$ & $80.7$\\
$2048$ & $88.0$ & $87.6$ & $87.4$ & $85.0$ & $85.3$\\
\bottomrule
\end{tabular}
\end{minipage}\hfill
\begin{minipage}[t]{0.32\linewidth}\centering
$r{=}4$\\[2.5pt]
\begin{tabular}{r r rr rr}
\toprule
& & \multicolumn{2}{c}{4-bit} & \multicolumn{2}{c}{2-bit}\\
\cmidrule(lr){3-4}\cmidrule(lr){5-6}
$b$ & bf16 & \textsc{post} & \textsc{pre} & \textsc{post} & \textsc{pre}\\
\midrule
$64$   & $68.5$ & $68.6$ & $68.8$ & $64.9$ & $63.7$\\
$256$  & $84.9$ & $83.8$ & $84.4$ & $80.4$ & $79.0$\\
$2048$ & $87.7$ & $87.2$ & $87.6$ & $84.5$ & $84.9$\\
\bottomrule
\end{tabular}
\end{minipage}\hfill
\begin{minipage}[t]{0.32\linewidth}\centering
$r{=}2$\\[2.5pt]
\begin{tabular}{r r rr rr}
\toprule
& & \multicolumn{2}{c}{4-bit} & \multicolumn{2}{c}{2-bit}\\
\cmidrule(lr){3-4}\cmidrule(lr){5-6}
$b$ & bf16 & \textsc{post} & \textsc{pre} & \textsc{post} & \textsc{pre}\\
\midrule
$64$   & $66.2$ & $65.4$ & $65.9$ & $60.0$ & $60.5$\\
$256$  & $83.2$ & $82.7$ & $82.2$ & $78.2$ & $76.5$\\
$2048$ & $87.6$ & $87.5$ & $87.8$ & $84.8$ & $84.3$\\
\bottomrule
\end{tabular}
\end{minipage}
\end{table}

\FloatBarrier

\Needspace{8\baselineskip}
\subsection{Ablations}
\label{sec:eval-abl}
We vary rank and basis precision across four page sizes at a fixed
attended-token budget (Fig.~\ref{fig:quant-ablation}). Int8 nearly tracks
bf16, whereas int4 levels off at higher rank and int2 gains little
from added rank. At the same rank, larger pages have lower
recall, so increasing $B$ calls for a richer basis.

\begin{figure}[H]\centering
\includegraphics[width=\linewidth]{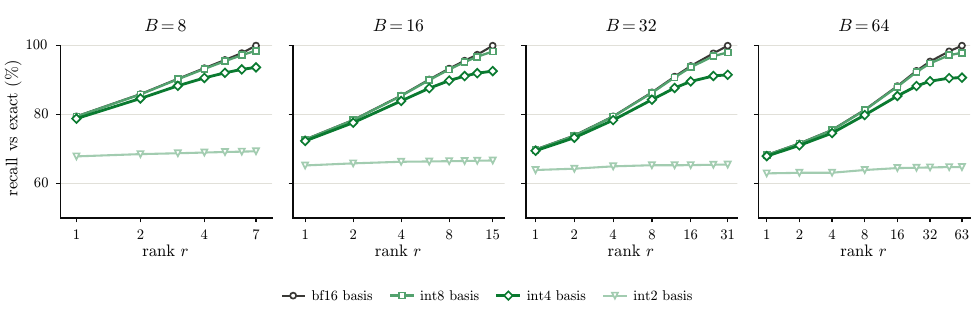}
\caption{Page recall vs.\ exact-LSE by rank, basis precision, and page
size (Llama-3.1-8B, $b{=}1024$).}
\label{fig:quant-ablation}
\end{figure}

Larger pages reduce summary bytes per KV token, but leave fewer pages
within the fixed budget. Table~\ref{tab:pagesize-main} shows how this
tradeoff reaches task scores: LongBench-v1 varies by at most $0.2$
points across page sizes, while RULER-32K favors $B{=}16$ by
$2.3$--$3.7$ points over the larger pages at these budgets. RULER-16K
is less sensitive, especially at $b{=}2048$. Further GQA ablations and
extended model studies appear in
App.~\ref{app:extstudies}.
\begin{table}[H]\centering
\caption{Page-size ablation (Llama-3.1-8B, r8i4); full grid and FullKV
in Table~\ref{tab:appC-pagesize}.}
\label{tab:pagesize-main}
\footnotesize
\setlength{\tabcolsep}{4.5pt}
\begin{tabular}{r rrr rrr rrr}
\toprule
& \multicolumn{3}{c}{LongBench-v1}
& \multicolumn{3}{c}{RULER-16K}
& \multicolumn{3}{c}{RULER-32K}\\
\cmidrule(lr){2-4}\cmidrule(lr){5-7}\cmidrule(lr){8-10}
$b$ & $B{=}16$ & $B{=}32$ & $B{=}64$
    & $B{=}16$ & $B{=}32$ & $B{=}64$
    & $B{=}16$ & $B{=}32$ & $B{=}64$\\
\midrule
$256$  & $46.7$ & $46.8$ & $46.8$ & $90.6$ & $89.7$ & $89.0$ & $85.5$ & $82.7$ & $83.2$\\
$1024$ & $47.2$ & $47.1$ & $47.1$ & $93.0$ & $92.1$ & $92.7$ & $87.2$ & $83.5$ & $84.8$\\
$2048$ & $47.2$ & $47.1$ & $47.0$ & $93.7$ & $93.0$ & $93.6$ & $88.0$ & $84.5$ & $85.1$\\
\bottomrule
\end{tabular}
\end{table}

\clearpage
\section{Related Work}
\label{sec:related}
\noindent\textbf{KV eviction.}
Eviction reduces cache residency by choosing what to keep before future
queries arrive. Methods use attention or reconstruction signals
\citep{zhang2023h2o,li2024snapkv,kim2025kvzip,zweiger2026fastcompaction},
sink-and-window structure~\citep{xiao2024streamingllm}, or head-specific
budgets and retention~\citep{feng2024adakv,yang2026hardkv,duoattention2024}.
RocketKV combines eviction with query-time selection~\citep{behnam2025rocketkv}.
These designs can reduce resident KV, but an evicted page is unavailable
to a later query. \method{} retains access to all pages while reading
only selected KV at each decode step.

\noindent\textbf{Sparse attention and retrieval.}
Query-aware selectors score retained context using page extrema
\citep{tang2024quest}, centroids or clusters
\citep{multipole2025,clusterkv2024}, and representative keys
\citep{xiao2024infllm}. Low-rank, latent, or partial-key representations also guide
retrieval~\citep{singhania2024loki,sun2025shadowkv,sals2025,lee2024infinigen,lrqk2025}.
Other selectors use codebooks or compressed keys
\citep{pqcache2024,selfindexkv2026}, moment summaries~\citep{cobs2026},
or head votes~\citep{tokenselect2024}.
\method{} differs in fitting each page's own basis and reconstructing
token-level logits before estimating its log-sum-exp mass. Related
systems train landmark-based chunk retrieval~\citep{mohtashami2023landmark,hils2026},
target coverage~\citep{lin2025twilight,tactic2025}, reuse rankings
\citep{yang2025tidaldecode,kascade2025}, or reduce decode scoring cost
\citep{ribar2024sparq,chen2025magicpig,liu2025retrievalattention}.
Prism scores spectral bands for prefill~\citep{prism2026}, while
MInference targets sparse prefill~\citep{jiang2024minference}.

\noindent\textbf{Selection objectives.}
The exact top-$k$ truncation identity also appears in
\citet{tzachristas2025topk} and \citet{cobs2026}; the latter derives a
value-agnostic mass criterion. Expected Attention, value-aware scoring, verified
sampling, and softmax thresholding offer other ways to decide what to
attend~\citep{expectedattention2025,vatp2024,desai2026vattention,blasst2026}.
Exact top-$k$ acceleration~\citep{gvr2026bitexact} and certified
quantized attention~\citep{calver2026quantcert} address the reliability
or cost of different selection mechanisms. Our target is the exact-LSE
ranking implied by value-blind output-error control (\S\ref{sec:setup}).

\noindent\textbf{KV quantization and compression.}
Quantizers shrink stored keys and values
\citep{liu2024kivi,hooper2024kvquant,svdq2025,zipcache2024,mikv2024}.
MQA and GQA share KV across query heads~\citep{shazeer2019mqa,ainslie2023gqa},
while low-rank schemes reduce each cached entry
\citep{palu2024,eigenattention2024,ojakv2025,kqsvd2025}. Storage
compression reduces cache footprint, while page selection reduces the
KV read per query. \S\ref{sec:eval-kivi} tests \method{} with KIVI on
accuracy.

\noindent\textbf{Trained sparsity and long reasoning.}
Native sparse and hybrid attention change the model's attention
architecture~\citep{yuan2025nsa,moba2025,deepseek2026v4,zai2026glm52,qwen2026qwen35}.
SeerAttention distills a gate~\citep{gao2024seerattention}, and SPLA
learns a selector inside sparse-plus-linear attention~\citep{spla2026}.
Long-reasoning methods instead allocate or evict generated history by
redundancy, hierarchy, values, or mass regions
\citep{cai2025rkv,reasonalloc2026,longflow2026,vase2026,amskv2026}.
\method{} uses one training-free page selector across long-input and
long-reasoning tasks. Sparse-attention tradeoffs across tasks are also
compared by \citet{nawrot2025sparsefrontier}.

\section{Conclusion}
\label{sec:conclusion}
Attention keys admit compact low-rank approximations within pages.
\method{} uses that locality to approximate exact page LSE, the
value-blind selection target linked to worst-case attention-output
error. At tight budgets, its training-free summaries track exact-LSE
selection on retrieval and QA while preserving long-form reasoning
accuracy.
On one GH200 with GPU-resident KV, \method{} r4i4 cuts complete
decode-step time by $1.8\times$ at $512$K context. On Qwen3-4B AIME26,
it also offers a better quality--latency tradeoff than Quest at similar
decode times (Fig.~\ref{fig:eff-glm-dense}(d)). With full KV offloaded
to Grace memory, its GPU hot cache admits larger batches and yields
$3.82$--$4.22\times$ the faster dense backend's aggregate throughput
at $64$K--$256$K context.
These results identify page-specific geometry as a useful unit of
representation for sparse decode. Adaptive page ranks and cheaper
summary scans remain open.

\clearpage
\section*{Reproducibility Statement}
The summary construction, page scoring, and selection rule are specified
in \S\ref{sec:method}, with proofs in App.~\ref{app:proofs}. The evaluation
protocols, model and baseline settings, budget accounting, hardware,
sampling seeds, and per-task results are reported in \S\ref{sec:eval}
and App.~\ref{app:extres}. Further implementation choices and ablations
are documented in App.~\ref{app:extstudies}.

\section*{AI Use Statement}
Generative AI tools assisted with literature search and citation checks,
manuscript drafting and editing, \LaTeX{} formatting, figure preparation
and plotting code, feedback on experimental design and protocols,
review of mathematical claims and proofs, and kernel optimizations.
The authors reviewed these suggestions, including the citations,
calculations, proofs, and
experimental reporting, and revised them where needed. The authors
take responsibility for the final text, claims, code, figures, and
other artifacts.

\bibliography{references}
\bibliographystyle{iclr2027_conference}

\clearpage
\appendix
\thispagestyle{plain}
\begin{center}
{\LARGE\bfseries Appendix}\\[2pt]
{\large\normalfont Table of Contents}
\end{center}
\vspace{0.4em}
\noindent\rule{\linewidth}{0.8pt}
\vspace{1.4em}

\newcommand{\tocentry}[3]{\noindent{\textbf{#1}\hspace{0.9em}\hyperref[#2]{#3}~\dotfill~\pageref{#2}}}
\newcommand{\subtocentry}[3]{\noindent{\itshape\hspace{2.0em}#1\hspace{0.8em}\hyperref[#2]{#3}~\dotfill~\pageref{#2}}}
\newcommand{\subsubtocentry}[3]{\noindent{\small\hspace{3.7em}#1\hspace{0.7em}\hyperref[#2]{#3}~\dotfill~\pageref{#2}}}

\tocentry{A}{app:proofs}{Proofs}\\[5pt]
\subtocentry{A.1}{appauto:A.1}{Coverage and minimax selection}\\[3pt]
\subtocentry{A.2}{appauto:A.2}{Shared projections and page-local bases}\\[3pt]
\subtocentry{A.3}{appauto:A.3}{Retention guarantee}\\[3pt]
\subtocentry{A.4}{appauto:A.4}{Group-shared selection}\\[3pt]
\subtocentry{A.5}{appauto:A.5}{Moment approximation and outlier capture}\\[10pt]
\tocentry{B}{app:extres}{Detailed Results}\\[5pt]
\subtocentry{B.1}{appauto:B.1}{Representation scope}\\[3pt]
\subtocentry{B.2}{appauto:B.2}{Retrieval and long-document QA}\\[2pt]
\subsubtocentry{B.2.1}{appauto:B.2.1}{Baselines}\\[2pt]
\subsubtocentry{B.2.2}{appauto:B.2.2}{LongBench-v1}\\[2pt]
\subsubtocentry{B.2.3}{appauto:B.2.3}{RULER-16K}\\[2pt]
\subsubtocentry{B.2.4}{appauto:B.2.4}{RULER-32K}\\[2pt]
\subsubtocentry{B.2.5}{appauto:B.2.5}{InfiniteBench}\\[3pt]
\subtocentry{B.3}{appauto:B.3}{Long-form reasoning}\\[2pt]
\subsubtocentry{B.3.1}{appauto:B.3.1}{Baselines}\\[2pt]
\subsubtocentry{B.3.2}{appauto:B.3.2}{AIME26}\\[2pt]
\subsubtocentry{B.3.3}{appauto:B.3.3}{MATH-500}\\[3pt]
\subtocentry{B.4}{appauto:B.4}{GH200 decode efficiency}\\[2pt]
\subsubtocentry{B.4.1}{appauto:B.4.1}{GPU-resident batch-1 decode}\\[2pt]
\subsubtocentry{B.4.2}{appauto:B.4.2}{Host-backed batched decode}\\[3pt]
\subtocentry{B.5}{appauto:B.5}{KV-cache quantization (KIVI): full grid}\\[10pt]
\tocentry{C}{app:extstudies}{Extended Studies}\\[5pt]
\subtocentry{C.1}{appauto:C.1}{GQA}\\[3pt]
\subtocentry{C.2}{appauto:C.2}{Rotary Position Embedding}\\[3pt]
\subtocentry{C.3}{appauto:C.3}{Page size}\\[3pt]
\subtocentry{C.4}{appauto:C.4}{Mixture-of-experts: ERNIE-4.5-21B-A3B-Thinking}\\[3pt]
\subtocentry{C.5}{appauto:C.5}{Large dense: DeepSeek-R1-Distill-Qwen-32B}\\[3pt]
\subtocentry{C.6}{appauto:C.6}{Learned sparse attention: Qwen3.8-Flash-Next}\\[3pt]
\subtocentry{C.7}{appauto:C.7}{H200 decode efficiency: measurement protocol and full numbers}\\[10pt]

\vspace{1.4em}
\noindent\rule{\linewidth}{0.4pt}

\makeatletter
\let\app@stdsection\section
\renewcommand{\section}[1]{\clearpage\app@stdsection{#1}%
  \setcounter{table}{0}\setcounter{figure}{0}}
\let\app@stdsubsection\subsection
\renewcommand{\subsection}[1]{%
  \Needspace{6\baselineskip}%
  \app@stdsubsection{#1}%
  \protected@edef\app@lbl{appauto:\thesubsection}%
  \expandafter\label\expandafter{\app@lbl}}
\let\app@stdsubsubsection\subsubsection
\renewcommand{\subsubsection}[1]{%
  \Needspace{5\baselineskip}%
  \app@stdsubsubsection{#1}%
  \protected@edef\app@lbl{appauto:\thesubsubsection}%
  \expandafter\label\expandafter{\app@lbl}}
\makeatother

\clubpenalty=10000
\widowpenalty=10000
\displaywidowpenalty=10000

\renewcommand{\thetable}{\thesection.\arabic{table}}
\renewcommand{\thefigure}{\thesection.\arabic{figure}}

\section{Proofs}
\label{app:proofs}
\subsection{Coverage and minimax selection}
\label{app:proofs-coverage}
\begin{lemma}[Coverage and deterministic value-blind selection]
\label{lem:mass}
Let $\|\mathbf v_i\|\le R$ for every cached value.
\par\noindent (i) For any nonempty page set $S$ and query $\mathbf q$,
let $\varepsilon=1-\gamma_S(\mathbf q)$. If $S^c$ is nonempty, the
mass-weighted mean values of the kept and dropped tokens satisfy
\[
\mathbf o(\mathbf q)-\mathbf o_S(\mathbf q)
=\varepsilon\bigl(\bar{\mathbf v}_{S^c}-\bar{\mathbf v}_S\bigr).
\]
In all cases,
$\|\mathbf o_S(\mathbf q)-\mathbf o(\mathbf q)\|\le2R\varepsilon$.
When $S^c$ is empty, $\mathbf o_S=\mathbf o$.
\par\noindent (ii) At a budget of $k$ pages, the top-$k$ pages by the
exact LSE score of Eq.~\eqref{eq:exact-page-lse} maximize coverage and
minimize the bound in~(i). They are minimax-optimal for worst-case
output error among deterministic value-blind selection rules, with
ties broken arbitrarily.
\end{lemma}
Throughout, $i\in S$ means $i\in\bigcup_{j\in S}P_j$ for a page set $S$.
For $i\in P_j$, write
$\boldsymbol\delta_i=\mathbf k_i-\boldsymbol\mu_j$.

\paragraph{Proof of Lemma~\ref{lem:mass}(i).}
Let $p_i=\alpha_i$ be the softmax weights over all tokens,
$Z_S=\sum_{i\in S}p_i=\gamma_S$. If $S^c$ is empty, $Z_S=1$ and
$\mathbf o_S=\mathbf o$, so the bound holds. Otherwise, define
$\bar{\mathbf v}_S=\tfrac1{Z_S}\sum_{i\in S}p_i\mathbf v_i$,
$\bar{\mathbf v}_{S^c}=\tfrac1{1-Z_S}\sum_{i\notin S}p_i\mathbf v_i$. The
renormalized restricted output is $\mathbf o_S=\bar{\mathbf v}_S$, while
$\mathbf o=Z_S\bar{\mathbf v}_S+(1-Z_S)\bar{\mathbf v}_{S^c}$. Subtracting,
$\mathbf o-\mathbf o_S=(1-Z_S)(\bar{\mathbf v}_{S^c}-\bar{\mathbf v}_S)
=\varepsilon(\bar{\mathbf v}_{S^c}-\bar{\mathbf v}_S)$.
The mass-weighted means are convex combinations of
the $\mathbf v_i$, so $\|\bar{\mathbf v}_{S^c}\|,\|\bar{\mathbf v}_S\|\le R$
and $\|\mathbf o-\mathbf o_S\|\le 2R\varepsilon=2R(1-\gamma_S)$. \qed

\begin{definition}[Deterministic value-blind selection rule]
\label{def:valueblind}
Fix a budget $k$. A selection rule is a map from the keys and query,
$(\{\mathbf k_i\},\mathbf q)$, to a $k$-subset of pages. It is
\emph{deterministic value-blind} if its output does not depend on the
cached values $\{\mathbf v_i\}$: two problem instances with identical
keys and query but different values receive the same selected set.
\end{definition}

\paragraph{Proof of Lemma~\ref{lem:mass}(ii).}
Coverage $\gamma_S=\sum_{j\in S}m_j$ with page-mass share
$m_j=e^{s_j}/\sum_{j'}e^{s_{j'}}$. Maximizing an additive set function over
$k$-subsets is attained by the $k$ largest $m_j$, i.e.\ the top-$k$ by $s_j$;
with mandatory pages fixed, take the largest remaining masses in the free
slots. This minimizes $\varepsilon=1-\gamma_S$ and the bound of part~(i).
If $k$ equals the number of pages, the unique set has zero error.
Otherwise, for minimax optimality among the rules of
Definition~\ref{def:valueblind}, fix
the keys and the query, so each candidate set $S$ determines its dropped
mass $\varepsilon_S$, independent of any rule's knowledge of
$\{\mathbf v_i\}$. A deterministic value-blind rule picks $S$ without
seeing $\{\mathbf v_i\}$, and the exact identity of part~(i) gives error
$\varepsilon_S\|\bar{\mathbf v}_{S^c}-\bar{\mathbf v}_S\|$. Over assignments
with $\max_i\|\mathbf v_i\|\le R$ this never exceeds $2R\varepsilon_S$, and
the assignment $\mathbf v_i=R\mathbf u$ on $S^c$, $\mathbf v_i=-R\mathbf u$ on
$S$ (any unit $\mathbf u$) attains it, so any fixed deterministic
value-blind rule's worst-case error at the set it picks is exactly
$2R\varepsilon_S$. It is minimized by minimizing $\varepsilon_S$ via the mass
top-$k$. \qed

\paragraph{Randomized selectors.}
For three equal-mass single-token pages and $k=1$, any deterministic
choice has worst-case error $4R/3$. Uniform random selection has
expected error $\frac13\sum_i\|\mathbf v_i-\bar{\mathbf v}\|\le R$,
since $\sum_i\|\mathbf v_i-\bar{\mathbf v}\|^2
=\sum_i\|\mathbf v_i\|^2-3\|\bar{\mathbf v}\|^2\le3R^2$.
Thus the minimax statement in Lemma~\ref{lem:mass} requires deterministic
selection.

\subsection{Shared projections and page-local bases}
\label{app:proofs-blindness}
\begin{lemma}[Envelope content-blindness]
\label{lem:env}
The envelope score of \citet{tang2024quest}, divided by $\sqrt d$,
upper-bounds every within-page
logit, but depends on a page only through its per-dimension key extremes:
pages with equal extremes receive identical scores for every query, and
for $d\ge2$ and $B\ge2$ there are equal-extreme page pairs whose exact log-masses
separate without bound.
\end{lemma}

\paragraph{Proof of Lemma~\ref{lem:env}.}
Let $u_{j,c}=\max_{i\in P_j}k_{i,c}$ and
$\ell_{j,c}=\min_{i\in P_j}k_{i,c}$. The envelope score is
\[
e_j(\mathbf q)=\sum_{c=1}^{d}
\max\bigl(q_c u_{j,c},\,q_c\ell_{j,c}\bigr).
\]
For any $i\in P_j$, $k_{i,c}\in[\ell_{j,c},u_{j,c}]$ implies
$q_ck_{i,c}\le\max(q_cu_{j,c},q_c\ell_{j,c})$ in every coordinate.
Summing and dividing by $\sqrt d$ upper-bounds each token logit.
Equal extremes give equal envelope scores for every query.

For separation, take $t>0$. Let $P'$ contain $(t,t,0,\dots)$ and
$(-t,-t,0,\dots)$, and let $P$ contain $(t,-t,0,\dots)$ and
$(-t,t,0,\dots)$; pad both to $B$ keys with zeros. Their extremes
agree, but for $\mathbf q=(\mathbf e_1+\mathbf e_2)/\sqrt2$ every key
of $P$ has zero logit, whereas $P'$ has one with logit
$t\sqrt2/\sqrt d$. Thus
$s_{P'}(\mathbf q)-s_P(\mathbf q)\ge t\sqrt2/\sqrt d-\log B
\to\infty$. \qed

\paragraph{Proof of Proposition~\ref{prop:blind}.}
Choose a unit $\mathbf u\perp\operatorname{span}(\mathbf W)$. Let $P$
contain $B$ zero keys, and let $P_c$ replace two of them by
$c\,\mathbf u$ and $-c\,\mathbf u$. Both centroids and all projected keys
are zero, so any scorer of the query, centroid, and projected keys gives
the pages the same score. For $\mathbf q=\mathbf u$,
$s_P(\mathbf q)=\log B$ while
$s_{P_c}(\mathbf q)=\log(B-2+2\cosh(c/\sqrt d))\to\infty$.
The centered keys of $P_c$ lie in $\operatorname{span}(\mathbf u)$; a
page-specific rank-one basis containing $\mathbf u$ reconstructs them
exactly. The zero page is exact with any basis. \qed

\begin{corollary}[The per-page basis is residual-optimal]
\label{cor:ey}
Fix a page $j$ and rank $r<d$. Among summaries
$(\boldsymbol\mu_j,\mathbf H,\{\mathbf H^{\top}\boldsymbol\delta_i\})$
over orthonormal $\mathbf H\in\R^{d\times r}$, the spectral basis
$\mathbf V_j$ minimizes the squared reconstruction residual
$\tau_{j,r}(\mathbf H)=\sum_{i\in P_j}\lVert\boldsymbol\delta_i
-\mathbf H\mathbf H^{\top}\boldsymbol\delta_i\rVert^2$. When
$\sigma_{j,r}>\sigma_{j,r+1}$ the minimizing \emph{subspace}
$\operatorname{span}(\mathbf H)=\operatorname{span}(\mathbf V_j)$ is
unique. Every orthonormal basis of that subspace attains the same
residual and Frobenius-tail bound of Prop.~\ref{prop:err}.
\end{corollary}

\paragraph{Proof of Corollary~\ref{cor:ey}.}
For any orthonormal $\mathbf H\in\R^{d\times r}$,
\[
\tau_{j,r}(\mathbf H)
=\lVert\mathbf D_j-\mathbf D_j\mathbf H\mathbf H^{\top}\rVert_F^2
\ge\sum_{\ell>r}\sigma_{j,\ell}^2.
\]
The Eckart--Young theorem \citep{eckart1936approximation} gives this
rank-$r$ minimum, and $\mathbf H=\mathbf V_j$ attains it. A strict gap
$\sigma_{j,r}>\sigma_{j,r+1}$ makes the leading $r$-dimensional
eigenspace of $\mathbf D_j^{\top}\mathbf D_j$ the unique minimizing
subspace. Rotating an orthonormal basis within it leaves the projector
$\mathbf H\mathbf H^{\top}$, and hence the residual, unchanged.

For the Frobenius-tail bound, set
$\mathbf q_{\perp,\mathbf H}:=(\mathbf I-\mathbf H\mathbf H^{\top})\mathbf q$.
Since $r<d$,
\[
\sup_{\|\mathbf q\|\le1}
\frac{\|\mathbf q_{\perp,\mathbf H}\|
\sqrt{\tau_{j,r}(\mathbf H)}}{\sqrt d}
=\frac{\sqrt{\tau_{j,r}(\mathbf H)}}{\sqrt d}.
\]
Thus minimizing the residual also minimizes this worst-case bound. \qed

\subsection{Retention guarantee}
\label{app:proofs-retain}

\paragraph{Step 1: per-query, per-page score error.}
\begin{proposition}[Score error characterization]
\label{prop:err}
Fix $\mathbf q$ and a page $j$; let $\hat s_j$ be the score of
Eq.~\eqref{eq:score} with fully unquantized summaries,
$\boldsymbol\delta^{\perp}_i=\boldsymbol\delta_i-\mathbf V_j\mathbf c_i$
the reconstruction residuals, and
$\mathbf q_{\perp}=\mathbf q-\mathbf V_j\mathbf V_j^{\top}\mathbf q$. Then
\[
\bigl|\hat s_j(\mathbf q)-s_j(\mathbf q)\bigr|
\;\le\;\max_{i\in P_j}
\frac{\bigl|\mathbf q_{\perp}^{\top}\boldsymbol\delta^{\perp}_i\bigr|}{\sqrt d}
\;\le\;\frac{\lVert\mathbf q_{\perp}\rVert\,\sigma_{j,r+1}}{\sqrt d}
\;\le\;\frac{\lVert\mathbf q_{\perp}\rVert\,\sqrt{\tau_{j,r}}}{\sqrt d},
\]
and $\hat s_j=s_j$ whenever $r\ge\operatorname{rank}(\mathbf D_j)$.
\end{proposition}

\paragraph{Proof.}
For each $i\in P_j$ the reconstructed logit differs from the true one by
\[
\frac{\mathbf q^{\top}(\mathbf V_j\mathbf c_i-
\boldsymbol\delta_i)}{\sqrt d}
=-\frac{\mathbf q_{\perp}^{\top}
\boldsymbol\delta^{\perp}_i}{\sqrt d},
\]
using $\mathbf V_j^{\top}\boldsymbol\delta^{\perp}_i=\mathbf 0$.
Log-sum-exp is
$1$-Lipschitz in the sup norm of its arguments (its gradient is a softmax,
hence a convex combination), which gives the first inequality. For the
second, $(\boldsymbol\delta^{\perp}_i)^{\top}$ is the $i$-th row of
$\mathbf D_j(\mathbf I-\mathbf V_j\mathbf V_j^{\top})$, so
$\lVert\boldsymbol\delta^{\perp}_i\rVert
\le\lVert\mathbf D_j(\mathbf I-\mathbf V_j\mathbf V_j^{\top})\rVert_2
=\sigma_{j,r+1}$ by optimality of the truncated SVD in spectral norm,
and Cauchy--Schwarz completes it. The weaker third form follows from
\[
\lVert\boldsymbol\delta^{\perp}_i\rVert^2
\le\sum_{i'\in P_j}\lVert\boldsymbol\delta^{\perp}_{i'}\rVert^2
=\sum_{\ell>r}\sigma_{j,\ell}^2=\tau_{j,r}.
\]
If
$r\ge\operatorname{rank}(\mathbf D_j)$, every residual vanishes and
$\hat s_j=s_j$. \qed

Taking the maximum over queries and pages proves
Theorem~\ref{thm:retain}(i).

\paragraph{Step 2: the group-share sandwich.}
Write $\eta=\max_{g\le G}\max_j|\hat s_j(\mathbf q_g)-s_j(\mathbf q_g)|$ as
above. Termwise, $e^{\hat s_j(\mathbf q_g)}\in[e^{-\eta},e^{\eta}]\cdot
e^{s_j(\mathbf q_g)}$ for every page $j$ and every query $\mathbf q_g$ in
the group; normalizing (which multiplies numerator and denominator by
factors in the same range) doubles the exponent, giving
$\hat m_g(j)\in[e^{-2\eta},e^{2\eta}]\cdot m_g(j)$ for every query in the
group. Averaging a family of ratios each sandwiched in
$[e^{-2\eta},e^{2\eta}]$ preserves the sandwich on the average,
$\bar{\hat m}_j=\tfrac1G\sum_g\hat m_g(j)\in[e^{-2\eta},e^{2\eta}]\cdot
\tfrac1G\sum_g m_g(j)=[e^{-2\eta},e^{2\eta}]\cdot\bar m_j$, i.e.
\[
\bigl|\log\bar{\hat m}_j-\log\bar m_j\bigr|\le2\eta
\qquad\text{for every page }j.
\]

\paragraph{Step 3: robustness to bounded score error.}
\label{app:proofs-robust}
\begin{corollary}[Selection is robust to any bounded score error]
\label{cor:robust}
Let $w_j>0$ be a weight for page $j$, and let $g_j$
be ranking scores satisfying $\max_j|g_j-\log w_j-\zeta|\le\Delta$ for some
page-independent offset $\zeta$. For $J\ge k$, fix a mandatory page set
$F$ with $|F|\le k$. Let $S_g$ and $S^\star$ both contain $F$ and fill their
remaining slots with the largest $g_j$ and $w_j$, respectively, among
pages outside $F$. Then
$\sum_{j\in S_g}w_j\ge e^{-2\Delta}
\sum_{j\in S^\star}w_j$.
\end{corollary}

\paragraph{Proof.}
Subtracting $\zeta$ from every $g_j$ changes no selected set, so take
$\zeta=0$.
Since $S_g$ maximizes $\sum_{p\in S}e^{g_p}$ among $k$-page sets containing
$F$,
\[
\sum_{p\in S_g}w_p
\;\ge\; e^{-\Delta}\!\sum_{p\in S_g}e^{g_p}
\;\ge\; e^{-\Delta}\!\sum_{p\in S^\star}e^{g_p}
\;\ge\; e^{-2\Delta}\!\sum_{p\in S^\star}w_p,
\]
where the outer inequalities hold termwise from
$|g_j-\log w_j|\le\Delta$. Dividing by $\sum_j w_j$ gives the same
ratio for normalized weights.
\qed

Apply Cor.~\ref{cor:robust} with $w_j=\bar m_j$,
$g_j=\log\bar{\hat m}_j$, $F$ equal to the sink and most recent pages,
and $\Delta=2\eta$ from Step~2. For any shared set $S$,
$\sum_{j\in S}\bar m_j=\tfrac1G\sum_g\gamma_S(\mathbf q_g)
=\bar\gamma(S)$. Cor.~\ref{cor:comb} identifies the exact-share
maximizer $S^\star$ within the same family of sets, so
$\bar\gamma(S_g)\ge e^{-4\eta}\bar\gamma(S^\star)
=e^{-4\eta}\bar\gamma^\star$, proving Theorem~\ref{thm:retain}(ii).

\paragraph{Step 4: from coverage to output error.}
Lemma~\ref{lem:mass}(i) applies to every query in the group against the
one shared selection $S$. Every query $\mathbf q_g$ reads the same KV
head, hence the same cached values and the same bound $R$, so
$\lVert\mathbf o_S(\mathbf q_g)-\mathbf o(\mathbf q_g)\rVert\le
2R(1-\gamma_S(\mathbf q_g))$ for every $g$. Averaging over the group,
\[
\frac1G\sum_{g\le G}\bigl\lVert\mathbf o_S(\mathbf q_g)-\mathbf o(\mathbf
q_g)\bigr\rVert
\;\le\;2R\Bigl(1-\frac1G\sum_{g\le G}\gamma_S(\mathbf q_g)\Bigr)
\;=\;2R\bigl(1-\bar\gamma(S)\bigr).
\]
Substituting the Step~3 bound gives Theorem~\ref{thm:retain}(iii). \qed

\subsection{Group-shared selection}
\label{app:proofs-combine}
\begin{corollary}[Group-shared selection]
\label{cor:comb}
Let $m_g(j)$ be query $\mathbf q_g$'s exact normalized page-mass share
and $\bar m_j=\tfrac1G\sum_g m_g(j)$ its group average. For $k\ge2$, among
shared sets of $k$ pages that always include the sink and most recent page,
the top-$(k{-}2)$ by $\bar m_j$ over the remaining pages maximizes the
equal-weight group-average coverage $\tfrac1G\sum_g\gamma_S(\mathbf
q_g)$, hence minimizes the correspondingly averaged bound of
Lemma~\ref{lem:mass}(i).
\end{corollary}

\paragraph{Proof.}
Let $Z_g=\sum_j e^{s_j(\mathbf q_g)}$ and
$m_g(j)=e^{s_j(\mathbf q_g)}/Z_g$. For any shared set $S$,
\[
\frac1G\sum_g\gamma_S(\mathbf q_g)
=\frac1G\sum_g\sum_{j\in S}m_g(j)
=\sum_{j\in S}\bar m_j.
\]
The mandatory sink and most recent pages contribute the same amount to
every candidate set. The largest remaining $k{-}2$ shares therefore
maximize group-average coverage. All queries in the group share one KV
head and value-norm bound $R$; averaging Lemma~\ref{lem:mass}(i) gives
$2R(1-\sum_{j\in S}\bar m_j)$, minimized by the same set. \qed

\subsection{Moment approximation and outlier capture}
\label{app:proofs-moment}
Let $\boldsymbol\phi_{j,\ell}$ denote the right singular vectors of
$\mathbf D_j$. For a completed page $|P_j|=B$, write
$a_j=\mathbf q^{\top}\boldsymbol\mu_j/\sqrt d$ and
$z_i=\mathbf q^{\top}\boldsymbol\delta_i/\sqrt d$. Its exact score is
\begin{equation}
s_j(\mathbf q)=a_j+\log\!\sum_{i\in P_j} e^{z_i}
 =a_j+\log B+\log\,\mathbb E_i\!\big[e^{z_i}\big],
\label{eq:cgf}
\end{equation}
where $\mathbb E_i$ averages over the page's keys.

\begin{proposition}[Quadratic approximation to page LSE]
\label{prop:quad}
Fix a query $\mathbf q$ and page $j$; let $\{z_i\}$ be the centered per-key
logits of Eq.~\eqref{eq:cgf} with empirical cumulants $\kappa_m$, and let
$L_j=\max_{i\in P_j}|z_i|$. Remove the page-independent $\log B$
and define the rank-$r$ quadratic score
\[
\hat s_j^{\mathrm{quad},r}
:=a_j+\frac{1}{2Bd}\sum_{\ell\le r}
\sigma_{j,\ell}^2(\boldsymbol\phi_{j,\ell}^{\top}\mathbf q)^2.
\]
With
\[
R^{\mathrm{cum}}_j:=\log\mathbb E_i[e^{z_i}]-\tfrac12\kappa_2,
\qquad
R^{\mathrm{rank}}_j:=\frac{1}{2Bd}\sum_{\ell>r}
\sigma_{j,\ell}^2(\boldsymbol\phi_{j,\ell}^{\top}\mathbf q)^2,
\]
the decomposition
$s_j(\mathbf q)-\log B
=\hat s_j^{\mathrm{quad},r}+R^{\mathrm{cum}}_j+R^{\mathrm{rank}}_j$
is exact. The remainders obey
\[
|R^{\mathrm{cum}}_j|\le L_j^3/3,
\qquad
R^{\mathrm{rank}}_j\le
\sigma_{j,r+1}^2\,
\|\mathbf q\|^2/(2Bd).
\]
\end{proposition}

\paragraph{Proof of Proposition~\ref{prop:quad}.}
By Eq.~\eqref{eq:cgf}, $K(t)\coloneqq\log\mathbb E_i[e^{tz_i}]$ is the cumulant
generating function of the empirical deviations $\{z_i\}$, finite and smooth
for all real $t$, with $K(0)=0$, $K'(0)=\kappa_1=0$ (the page is centered),
$K''(0)=\kappa_2$, and $s_j(\mathbf q)-a_j-\log B=K(1)$. Taylor's theorem with
integral remainder at $t=1$ gives
\[
K(1)=\tfrac12\kappa_2
 +\int_0^1\tfrac{(1-t)^2}{2}\,K'''(t)\,dt.
\]
Here $K'''(t)$ is the third cumulant of the exponentially tilted law
$\pi_t(i)\propto e^{tz_i}$, which is supported on the same deviations. Under it
$|z|\le L_j$, so its variance is at most $L_j^2$ and, since
$|z-\mathbb E_t z|\le 2L_j$ pointwise,
$|K'''(t)|=|\mathbb E_t(z-\mathbb E_t z)^3|\le
2L_j\,\mathrm{Var}_t(z)\le 2L_j^3$. Integrating,
$|R^{\mathrm{cum}}_j|\le 2L_j^3\int_0^1\tfrac{(1-t)^2}{2}\,dt=L_j^3/3$.
For the rank term,
$\kappa_2=\mathbb E_i z_i^2=\mathbf q^{\top}\mathbf S_j\mathbf q/(Bd)$ with
$\mathbf S_j=\sum_\ell\sigma_{j,\ell}^2\boldsymbol\phi_{j,\ell}\boldsymbol\phi_{j,\ell}^{\top}$.
Separating the top $r$ terms of $\mathbf q^\top\mathbf S_j\mathbf q$
gives the stated rank remainder. Since the $\boldsymbol\phi_{j,\ell}$
are orthonormal,
$\sum_{\ell>r}(\boldsymbol\phi_{j,\ell}^{\top}\mathbf q)^2
\le\|\mathbf q\|^2$; multiplying by the largest omitted singular
value gives
$R^{\mathrm{rank}}_j\le\sigma_{j,r+1}^2\|\mathbf q\|^2/(2Bd)$. \qed

\begin{lemma}[Outlier capture]
\label{lem:outlier}
Fix $\mathbf q$, a page $j$ with centered deviations $\{\boldsymbol\delta_i\}$,
and $1\le r<d$. For every $i^\star\in P_j$,
\[
\begin{aligned}
\sum_{\ell\le r}\sigma_{j,\ell}^2
(\boldsymbol\phi_{j,\ell}^{\top}\mathbf q)^2
&\ge(\mathbf q^{\top}\boldsymbol\delta_{i^\star})^2
-\sigma_{j,r+1}^2\|\mathbf q\|^2,\\
\sigma_{j,r+1}^2
&\le\lambda_r\Bigl(
\textstyle\sum_{i\ne i^\star}
\boldsymbol\delta_i\boldsymbol\delta_i^{\top}\Bigr),
\end{aligned}
\]
with $\lambda_r$ the $r$-th largest eigenvalue.
\end{lemma}

\paragraph{Proof of Lemma~\ref{lem:outlier}.}
Write $\mathbf S_j=
\boldsymbol\delta_{i^\star}\boldsymbol\delta_{i^\star}^{\top}
+\mathbf E$, where
$\mathbf E=\sum_{i\ne i^\star}
\boldsymbol\delta_i\boldsymbol\delta_i^{\top}\succeq0$.
The omitted singular values are at most $\sigma_{j,r+1}$, so
\[
\begin{aligned}
\sum_{\ell\le r}\sigma_{j,\ell}^2
(\boldsymbol\phi_{j,\ell}^{\top}\mathbf q)^2
&=\mathbf q^{\top}\mathbf S_j\mathbf q
-\sum_{\ell>r}\sigma_{j,\ell}^2
(\boldsymbol\phi_{j,\ell}^{\top}\mathbf q)^2\\
&\ge\mathbf q^{\top}\mathbf S_j\mathbf q
-\sigma_{j,r+1}^2\|\mathbf q\|^2.
\end{aligned}
\]
The first claim follows from
$\mathbf S_j\succeq
\boldsymbol\delta_{i^\star}\boldsymbol\delta_{i^\star}^{\top}$.
For the second, Weyl's inequality for this rank-one split gives
\[
\sigma_{j,r+1}^2=\lambda_{r+1}(\mathbf S_j)
\le\lambda_2(\boldsymbol\delta_{i^\star}
\boldsymbol\delta_{i^\star}^{\top})+\lambda_r(\mathbf E)
=\lambda_r(\mathbf E).
\]
The rank-one term has zero second eigenvalue. \qed

\section{Detailed Results}
\label{app:extres}
\FloatBarrier
This appendix gives the representation-scope controls of \S\ref{sec:law},
the per-task scores, and the GH200 efficiency protocol behind
\S\ref{sec:eval}.
\emph{Oracle} uses exact-LSE page scores and the share-average GQA
selector (App.~\ref{app:proofs-combine}).

\subsection{Representation scope}
\label{app:reprscope}
All controls use post-RoPE keys from the same RULER-32K records,
with $B{=}16$ and the shared GQA selector. At $b{=}128$, eight page slots
are available per KV head; the sink and current pages occupy two of them.
At $b{=}256$, fourteen of the sixteen slots remain for selection.

For the hybrid controls, each completed page retains its centroid
$\boldsymbol\mu_j$ from \S\ref{sec:law}, while the shared basis contains
the top $r$ eigenvectors of $\sum_{j\in\mathcal S}\mathbf H_j$.
The sequence fit pools completed prompt pages and freezes after prefill.
Global fits pool page-centered residuals from the independent
WikiText-2 raw validation split~\citep{merity2016pointer} in 4K or
32K windows and freeze before evaluation.
The fully shared controls instead fit one centroid and basis to the
original keys of the sequence or calibration set. Their sequence fit
includes the incomplete final prompt page; the hybrid fit does not.

All arms use int4 bases and int8 centroids and coefficients with bf16
scales. Within the r4 storage tier, page-local, hybrid, and fully shared
fits use ranks $4$, $20$, and $28$, respectively, at $490.0$, $487.8$,
and $488.4$ selector bytes per page and KV head. The corresponding r8
tier uses ranks $8$, $40$, and $48$ at $818.0$, $813.5$, and $814.1$
bytes. The global 4K and 32K labels specify calibration-window length,
not RULER context length. Table~\ref{tab:appB-scope} gives both tiers at
$b\in\{128,256\}$.
\begingroup\footnotesize\setlength{\tabcolsep}{5pt}
\renewcommand{\arraystretch}{0.85}
\begin{longtable}{@{}llrrrr@{}}
\caption{Representation-scope results on RULER-32K. The scope columns
give the centroid and basis fits; Page/Page is \method{}. Mass, recall,
and task scores are percentages.}\label{tab:appB-scope}\\
\toprule
Centroid & Basis & Rank & Mass & Recall & Score \\
\midrule\endfirsthead
\multicolumn{6}{c}{\tablename\ \thetable\ (cont.)}\\
\toprule
Centroid & Basis & Rank & Mass & Recall & Score \\
\midrule\endhead
\bottomrule\endfoot
\multicolumn{6}{l}{\emph{r4 storage tier, $b{=}128$}}\\
\midrule
Page & Page & 4 & \textbf{85.8} & \textbf{90.5} & \textbf{80.1} \\
Page & Sequence & 20 & 84.4 & 81.3 & 77.0 \\
Page & Global 4K & 20 & 84.6 & 81.9 & 76.3 \\
Page & Global 32K & 20 & 84.5 & 81.9 & 75.2 \\
Sequence & Sequence & 28 & 81.2 & 60.2 & 47.6 \\
Global 4K & Global 4K & 28 & 81.3 & 58.0 & 63.1 \\
Global 32K & Global 32K & 28 & 82.2 & 63.3 & 56.9 \\
\midrule
\multicolumn{6}{l}{\emph{r4 storage tier, $b{=}256$}}\\
\midrule
Page & Page & 4 & \textbf{87.6} & \textbf{86.1} & \textbf{84.9} \\
Page & Sequence & 20 & 86.3 & 73.7 & 81.4 \\
Page & Global 4K & 20 & 86.5 & 74.2 & 81.4 \\
Page & Global 32K & 20 & 86.4 & 74.4 & 81.1 \\
Sequence & Sequence & 28 & 83.6 & 51.6 & 65.1 \\
Global 4K & Global 4K & 28 & 83.6 & 45.6 & 70.0 \\
Global 32K & Global 32K & 28 & 84.4 & 53.7 & 72.4 \\
\midrule
\multicolumn{6}{l}{\emph{r8 storage tier, $b{=}128$}}\\
\midrule
Page & Page & 8 & \textbf{86.0} & \textbf{93.6} & \textbf{83.4} \\
Page & Sequence & 40 & 84.8 & 83.6 & 79.4 \\
Page & Global 4K & 40 & 85.1 & 84.2 & 80.0 \\
Page & Global 32K & 40 & 84.9 & 84.2 & 78.8 \\
Sequence & Sequence & 48 & 83.2 & 70.9 & 69.3 \\
Global 4K & Global 4K & 48 & 83.2 & 65.2 & 74.4 \\
Global 32K & Global 32K & 48 & 83.9 & 73.4 & 77.0 \\
\midrule
\multicolumn{6}{l}{\emph{r8 storage tier, $b{=}256$}}\\
\midrule
Page & Page & 8 & \textbf{87.7} & \textbf{90.6} & \textbf{85.5} \\
Page & Sequence & 40 & 86.7 & 77.1 & 80.0 \\
Page & Global 4K & 40 & 87.0 & 77.4 & 81.0 \\
Page & Global 32K & 40 & 86.8 & 77.8 & 82.2 \\
Sequence & Sequence & 48 & 85.4 & 62.8 & 74.8 \\
Global 4K & Global 4K & 48 & 85.2 & 52.6 & 79.7 \\
Global 32K & Global 32K & 48 & 85.9 & 64.5 & 79.6 \\
\end{longtable}
\endgroup

\FloatBarrier

\subsection{Retrieval and long-document QA}
\label{app:extres-retrieval}
\subsubsection{Baselines}
\label{app:extres-retrieval-baselines}
LongBench-v1 uses the full split; RULER uses the $16$K and $32$K
settings, all with Llama-3.1-8B. FullKV supplies the dense reference score.
The retrieval baselines are:
\begin{itemize}
  \setlength{\itemsep}{3pt}\setlength{\parskip}{0pt}\setlength{\topsep}{3pt}
  \item \textbf{Quest.} Our GQA port selects per query head. A group
        attends the union of its heads' pages, between $b$ and $Gb$
        tokens ($G{=}4$ for Llama-3.1-8B).
  \item \textbf{ShadowKV.} Its $384$-token outlier block and $32$-token
        local window count within $b$, leaving $96$, $608$, and $1632$
        tokens for additional selection at $b{=}512$, $1024$, and
        $2048$. At $b\le256$ the fixed block does not fit, so those
        dense-fallback points are omitted.
  \item \textbf{RocketKV.} Splits $b$ between attention and top-$k$
        estimation; attended tokens are $b/2$, measured decode traffic
        is $0.77$--$1.00\times$ nominal, and the full cache remains resident.
  \item \textbf{KVzip.} Retains the question and generated turn beyond
        $b$, so decode reads $b+|q|+g$ after $g$ generated tokens.
\end{itemize}

\subsubsection{LongBench-v1}
\begingroup\scriptsize\setlength{\tabcolsep}{2.6pt}
\begin{longtable}{lrrrrrrrrr}
\caption{LongBench-v1 scores by subset (Llama-3.1-8B), with all methods at each budget $b$.}\label{tab:appB-lbv1}\\
\toprule
Subtask & FullKV & Oracle & r8i4 & r4i4 & r2i4 & Quest & KVzip & ShadowKV & RocketKV \\
\midrule\endfirsthead
\multicolumn{10}{c}{\tablename\ \thetable\ (cont.)}\\
\toprule Subtask & FullKV & Oracle & r8i4 & r4i4 & r2i4 & Quest & KVzip & ShadowKV & RocketKV \\
\midrule\endhead
\bottomrule\endfoot
\midrule
\multicolumn{10}{l}{\emph{budget } $b=64$}\\
\midrule
narrativeqa & 29.4 & 28.7 & 29.3 & 28.0 & 28.4 & 17.3 & 21.2 & N/A & 26.3 \\
qasper & 47.2 & 44.9 & 45.5 & 45.9 & 43.7 & 27.3 & 21.1 & N/A & 28.0 \\
multifieldqa\_en & 55.4 & 56.6 & 56.0 & 55.3 & 56.4 & 35.8 & 34.7 & N/A & 55.9 \\
hotpotqa & 58.2 & 59.5 & 58.5 & 57.1 & 57.9 & 42.8 & 52.0 & N/A & 57.1 \\
2wikimqa & 51.9 & 50.9 & 50.8 & 50.6 & 50.5 & 36.0 & 44.8 & N/A & 48.4 \\
musique & 32.9 & 33.0 & 31.7 & 31.5 & 32.4 & 21.6 & 24.5 & N/A & 28.5 \\
dureader & 33.3 & 32.2 & 31.9 & 29.4 & 30.3 & 14.1 & 19.7 & N/A & 23.1 \\
gov\_report & 35.4 & 29.5 & 29.8 & 28.9 & 26.9 & 16.3 & 15.7 & N/A & 22.2 \\
qmsum & 24.8 & 24.2 & 24.1 & 23.9 & 24.0 & 18.0 & 20.4 & N/A & 22.5 \\
multi\_news & 27.2 & 25.9 & 25.8 & 25.1 & 24.8 & 18.2 & 18.6 & N/A & 23.6 \\
trec & 29.5 & 28.0 & 29.0 & 27.0 & 24.0 & 22.8 & 9.5 & N/A & 19.5 \\
triviaqa & 92.1 & 91.6 & 92.0 & 91.3 & 91.0 & 85.1 & 91.5 & N/A & 91.2 \\
samsum & 40.9 & 38.5 & 38.6 & 38.8 & 38.7 & 25.1 & 38.6 & N/A & 37.6 \\
passage\_retrieval\_en & 100.0 & 98.5 & 97.5 & 98.0 & 97.0 & 62.5 & 99.5 & N/A & 98.5 \\
\midrule
\textbf{Avg} & \textbf{47.0} & \textbf{45.9} & \textbf{45.7} & \textbf{45.1} & \textbf{44.7} & \textbf{31.6} & \textbf{36.6} & \textbf{N/A} & \textbf{41.6} \\
\midrule
\multicolumn{10}{l}{\emph{budget } $b=128$}\\
\midrule
narrativeqa & 29.4 & 28.7 & 28.8 & 28.2 & 28.2 & 22.0 & 21.9 & N/A & 26.7 \\
qasper & 47.2 & 46.5 & 46.6 & 46.6 & 47.2 & 37.5 & 27.6 & N/A & 40.3 \\
multifieldqa\_en & 55.4 & 57.1 & 56.6 & 56.4 & 55.9 & 46.3 & 40.5 & N/A & 57.1 \\
hotpotqa & 58.2 & 59.1 & 59.1 & 59.2 & 59.3 & 52.5 & 51.2 & N/A & 58.9 \\
2wikimqa & 51.9 & 52.8 & 51.0 & 50.5 & 51.0 & 47.0 & 45.7 & N/A & 51.0 \\
musique & 32.9 & 31.6 & 31.6 & 31.4 & 31.6 & 26.0 & 25.0 & N/A & 31.0 \\
dureader & 33.3 & 34.1 & 35.0 & 33.9 & 33.0 & 20.1 & 20.1 & N/A & 27.1 \\
gov\_report & 35.4 & 34.1 & 33.6 & 32.6 & 32.0 & 21.1 & 17.0 & N/A & 24.2 \\
qmsum & 24.8 & 24.1 & 24.5 & 24.6 & 24.0 & 19.8 & 21.0 & N/A & 23.3 \\
multi\_news & 27.2 & 26.3 & 26.3 & 26.5 & 26.2 & 24.2 & 23.6 & N/A & 24.0 \\
trec & 29.5 & 28.0 & 30.0 & 29.5 & 28.5 & 25.5 & 11.0 & N/A & 23.5 \\
triviaqa & 92.1 & 91.7 & 91.8 & 91.8 & 91.8 & 90.0 & 91.5 & N/A & 91.4 \\
samsum & 40.9 & 38.7 & 38.4 & 38.2 & 38.8 & 33.4 & 38.9 & N/A & 38.8 \\
passage\_retrieval\_en & 100.0 & 99.5 & 100.0 & 100.0 & 100.0 & 92.2 & 99.5 & N/A & 99.5 \\
\midrule
\textbf{Avg} & \textbf{47.0} & \textbf{46.6} & \textbf{46.7} & \textbf{46.4} & \textbf{46.2} & \textbf{39.8} & \textbf{38.2} & \textbf{N/A} & \textbf{44.1} \\
\midrule
\multicolumn{10}{l}{\emph{budget } $b=256$}\\
\midrule
narrativeqa & 29.4 & 28.6 & 28.1 & 28.1 & 29.0 & 26.0 & 23.2 & N/A & 29.6 \\
qasper & 47.2 & 45.8 & 46.5 & 46.2 & 46.1 & 44.0 & 39.5 & N/A & 43.8 \\
multifieldqa\_en & 55.4 & 56.9 & 56.8 & 56.3 & 56.0 & 51.9 & 50.5 & N/A & 56.9 \\
hotpotqa & 58.2 & 58.3 & 58.6 & 58.4 & 58.7 & 55.6 & 52.5 & N/A & 58.4 \\
2wikimqa & 51.9 & 52.5 & 51.6 & 51.7 & 51.5 & 49.3 & 48.5 & N/A & 51.9 \\
musique & 32.9 & 32.0 & 31.8 & 32.4 & 33.2 & 30.1 & 25.1 & N/A & 32.2 \\
dureader & 33.3 & 34.0 & 33.7 & 33.4 & 33.9 & 27.1 & 19.6 & N/A & 30.3 \\
gov\_report & 35.4 & 34.4 & 34.2 & 33.6 & 33.9 & 29.4 & 20.8 & N/A & 26.6 \\
qmsum & 24.8 & 24.4 & 25.0 & 24.7 & 24.7 & 22.0 & 22.2 & N/A & 24.5 \\
multi\_news & 27.2 & 27.1 & 26.8 & 26.8 & 26.6 & 26.7 & 25.8 & N/A & 24.5 \\
trec & 29.5 & 30.0 & 30.5 & 30.0 & 30.0 & 26.6 & 21.1 & N/A & 30.0 \\
triviaqa & 92.1 & 91.8 & 92.1 & 91.5 & 92.5 & 90.8 & 91.5 & N/A & 91.7 \\
samsum & 40.9 & 39.5 & 38.5 & 38.8 & 39.1 & 37.4 & 39.0 & N/A & 40.0 \\
passage\_retrieval\_en & 100.0 & 99.5 & 100.0 & 100.0 & 100.0 & 97.0 & 99.0 & N/A & 100.0 \\
\midrule
\textbf{Avg} & \textbf{47.0} & \textbf{46.8} & \textbf{46.7} & \textbf{46.6} & \textbf{46.8} & \textbf{43.8} & \textbf{41.3} & \textbf{N/A} & \textbf{45.7} \\
\midrule
\multicolumn{10}{l}{\emph{budget } $b=512$}\\
\midrule
narrativeqa & 29.4 & 29.0 & 29.1 & 29.1 & 28.3 & 28.5 & 24.9 & 29.5 & 29.9 \\
qasper & 47.2 & 45.9 & 46.6 & 47.2 & 46.6 & 45.3 & 44.4 & 45.6 & 46.4 \\
multifieldqa\_en & 55.4 & 55.6 & 56.0 & 56.4 & 56.8 & 52.8 & 54.0 & 56.7 & 58.3 \\
hotpotqa & 58.2 & 58.4 & 59.0 & 58.8 & 58.8 & 57.9 & 55.3 & 58.6 & 59.1 \\
2wikimqa & 51.9 & 51.8 & 51.8 & 52.1 & 51.9 & 51.7 & 49.8 & 51.0 & 50.2 \\
musique & 32.9 & 32.4 & 32.4 & 32.4 & 32.7 & 32.3 & 28.3 & 33.1 & 32.6 \\
dureader & 33.3 & 33.1 & 33.9 & 33.2 & 33.0 & 32.3 & 24.0 & 32.1 & 32.2 \\
gov\_report & 35.4 & 34.8 & 34.8 & 34.5 & 34.2 & 33.8 & 26.4 & 32.7 & 30.3 \\
qmsum & 24.8 & 24.9 & 24.9 & 24.6 & 24.6 & 24.4 & 23.6 & 24.8 & 25.1 \\
multi\_news & 27.2 & 27.2 & 27.2 & 27.2 & 27.2 & 27.7 & 26.8 & 26.8 & 25.8 \\
trec & 29.5 & 30.0 & 30.5 & 30.5 & 30.5 & 28.0 & 28.5 & 29.5 & 30.0 \\
triviaqa & 92.1 & 92.1 & 92.2 & 92.1 & 91.7 & 91.0 & 91.9 & 91.7 & 92.2 \\
samsum & 40.9 & 40.4 & 40.4 & 39.6 & 40.1 & 37.3 & 39.3 & 38.4 & 40.2 \\
passage\_retrieval\_en & 100.0 & 100.0 & 100.0 & 100.0 & 99.5 & 99.5 & 97.0 & 98.0 & 99.5 \\
\midrule
\textbf{Avg} & \textbf{47.0} & \textbf{46.8} & \textbf{47.1} & \textbf{47.0} & \textbf{46.8} & \textbf{45.9} & \textbf{43.9} & \textbf{46.3} & \textbf{46.6} \\
\midrule
\multicolumn{10}{l}{\emph{budget } $b=1024$}\\
\midrule
narrativeqa & 29.4 & 29.5 & 30.0 & 29.7 & 30.0 & 28.7 & 28.3 & 29.8 & 28.7 \\
qasper & 47.2 & 47.0 & 47.3 & 46.8 & 47.0 & 45.7 & 46.4 & 46.3 & 47.0 \\
multifieldqa\_en & 55.4 & 55.8 & 55.8 & 55.6 & 56.0 & 53.7 & 56.2 & 57.0 & 56.3 \\
hotpotqa & 58.2 & 58.8 & 59.6 & 59.6 & 59.0 & 57.5 & 58.9 & 59.2 & 59.0 \\
2wikimqa & 51.9 & 52.4 & 52.1 & 51.9 & 51.6 & 52.2 & 51.7 & 50.7 & 51.5 \\
musique & 32.9 & 32.5 & 33.0 & 32.8 & 32.6 & 32.8 & 32.2 & 33.7 & 31.6 \\
dureader & 33.3 & 33.9 & 33.2 & 33.3 & 33.1 & 33.5 & 31.0 & 33.2 & 32.6 \\
gov\_report & 35.4 & 35.2 & 35.0 & 34.9 & 34.5 & 34.8 & 31.7 & 33.9 & 32.7 \\
qmsum & 24.8 & 24.9 & 24.8 & 24.9 & 24.4 & 25.0 & 24.5 & 24.9 & 25.4 \\
multi\_news & 27.2 & 27.1 & 26.9 & 27.2 & 27.1 & 27.3 & 26.8 & 26.9 & 26.5 \\
trec & 29.5 & 30.5 & 30.0 & 30.0 & 30.0 & 29.0 & 30.0 & 30.0 & 29.5 \\
triviaqa & 92.1 & 92.1 & 92.1 & 92.2 & 92.2 & 91.1 & 91.9 & 91.8 & 92.2 \\
samsum & 40.9 & 40.2 & 40.5 & 40.6 & 40.3 & 39.5 & 39.2 & 40.6 & 40.3 \\
passage\_retrieval\_en & 100.0 & 100.0 & 100.0 & 100.0 & 99.5 & 99.5 & 91.0 & 99.5 & 100.0 \\
\midrule
\textbf{Avg} & \textbf{47.0} & \textbf{47.1} & \textbf{47.2} & \textbf{47.1} & \textbf{46.9} & \textbf{46.5} & \textbf{45.7} & \textbf{47.0} & \textbf{46.7} \\
\midrule
\multicolumn{10}{l}{\emph{budget } $b=2048$}\\
\midrule
narrativeqa & 29.4 & 30.3 & 30.2 & 30.0 & 30.3 & 29.7 & 29.1 & 29.6 & 30.1 \\
qasper & 47.2 & 47.0 & 47.3 & 47.4 & 47.2 & 46.2 & 47.6 & 46.4 & 46.9 \\
multifieldqa\_en & 55.4 & 54.8 & 55.0 & 54.9 & 54.9 & 55.3 & 56.4 & 56.2 & 55.0 \\
hotpotqa & 58.2 & 59.0 & 59.5 & 59.5 & 59.7 & 57.7 & 59.2 & 59.4 & 58.7 \\
2wikimqa & 51.9 & 51.7 & 51.6 & 51.2 & 51.0 & 52.4 & 51.5 & 50.7 & 51.6 \\
musique & 32.9 & 32.9 & 33.1 & 33.1 & 32.9 & 31.6 & 33.4 & 33.1 & 32.3 \\
dureader & 33.3 & 33.6 & 33.2 & 33.6 & 33.5 & 34.2 & 33.4 & 32.6 & 32.8 \\
gov\_report & 35.4 & 35.1 & 35.5 & 35.1 & 34.5 & 35.7 & 34.3 & 34.5 & 33.8 \\
qmsum & 24.8 & 25.3 & 25.1 & 25.2 & 24.7 & 25.3 & 24.9 & 25.1 & 24.9 \\
multi\_news & 27.2 & 27.1 & 27.1 & 26.9 & 26.9 & 26.9 & 26.9 & 26.9 & 27.2 \\
trec & 29.5 & 30.0 & 30.0 & 30.0 & 30.0 & 29.5 & 30.5 & 30.0 & 30.5 \\
triviaqa & 92.1 & 92.1 & 92.2 & 92.2 & 92.1 & 91.6 & 92.1 & 91.8 & 92.2 \\
samsum & 40.9 & 40.4 & 40.9 & 40.3 & 40.6 & 40.5 & 40.0 & 40.7 & 40.6 \\
passage\_retrieval\_en & 100.0 & 100.0 & 100.0 & 100.0 & 99.5 & 100.0 & 98.5 & 99.5 & 100.0 \\
\midrule
\textbf{Avg} & \textbf{47.0} & \textbf{47.1} & \textbf{47.2} & \textbf{47.1} & \textbf{47.0} & \textbf{46.9} & \textbf{47.0} & \textbf{46.9} & \textbf{46.9} \\
\end{longtable}
\endgroup

\subsubsection{RULER-16K}
\begingroup\scriptsize\setlength{\tabcolsep}{2.6pt}
\begin{longtable}{lrrrrrrrrr}
\caption{RULER-16K scores by task (Llama-3.1-8B), with all methods at each budget $b$.}\label{tab:appB-ruler16}\\
\toprule
Subtask & FullKV & Oracle & r8i4 & r4i4 & r2i4 & Quest & KVzip & ShadowKV & RocketKV \\
\midrule\endfirsthead
\multicolumn{10}{c}{\tablename\ \thetable\ (cont.)}\\
\toprule Subtask & FullKV & Oracle & r8i4 & r4i4 & r2i4 & Quest & KVzip & ShadowKV & RocketKV \\
\midrule\endhead
\bottomrule\endfoot
\midrule
\multicolumn{10}{l}{\emph{budget } $b=64$}\\
\midrule
NIAH-S1 & 100.0 & 100.0 & 100.0 & 100.0 & 100.0 & 100.0 & 100.0 & N/A & 100.0 \\
NIAH-S2 & 100.0 & 100.0 & 100.0 & 100.0 & 98.0 & 74.0 & 0.0 & N/A & 100.0 \\
NIAH-S3 & 100.0 & 100.0 & 100.0 & 100.0 & 96.0 & 4.0 & 0.0 & N/A & 100.0 \\
NIAH-MK1 & 100.0 & 100.0 & 100.0 & 98.0 & 96.0 & 88.0 & 0.0 & N/A & 100.0 \\
NIAH-MK2 & 100.0 & 98.0 & 98.0 & 78.0 & 80.0 & 12.0 & 0.0 & N/A & 98.0 \\
NIAH-MK3 & 98.0 & 70.0 & 44.0 & 14.0 & 4.0 & 0.0 & 0.0 & N/A & 86.0 \\
NIAH-MV & 98.5 & 79.5 & 82.5 & 77.0 & 67.0 & 26.5 & 0.0 & N/A & 44.0 \\
NIAH-MQ & 99.0 & 90.0 & 87.5 & 81.5 & 74.5 & 25.5 & 0.0 & N/A & 37.5 \\
VT & 100.0 & 82.8 & 78.4 & 76.8 & 76.0 & 59.6 & 88.4 & N/A & 20.8 \\
CWE & 87.4 & 5.0 & 6.6 & 6.0 & 3.0 & 6.0 & 0.2 & N/A & 0.8 \\
FWE & 96.7 & 75.3 & 78.0 & 76.0 & 72.0 & 44.0 & 0.0 & N/A & 76.7 \\
QA-1 & 84.0 & 82.0 & 82.0 & 82.0 & 82.0 & 44.0 & 56.0 & N/A & 84.0 \\
QA-2 & 62.0 & 58.0 & 58.0 & 52.0 & 50.0 & 30.0 & 38.0 & N/A & 54.0 \\
\midrule
\textbf{Avg} & \textbf{94.3} & \textbf{80.0} & \textbf{78.1} & \textbf{72.4} & \textbf{69.1} & \textbf{39.5} & \textbf{21.7} & \textbf{N/A} & \textbf{69.4} \\
\midrule
\multicolumn{10}{l}{\emph{budget } $b=128$}\\
\midrule
NIAH-S1 & 100.0 & 100.0 & 100.0 & 100.0 & 100.0 & 100.0 & 100.0 & N/A & 100.0 \\
NIAH-S2 & 100.0 & 100.0 & 100.0 & 100.0 & 100.0 & 88.0 & 2.0 & N/A & 100.0 \\
NIAH-S3 & 100.0 & 100.0 & 100.0 & 100.0 & 100.0 & 32.0 & 0.0 & N/A & 100.0 \\
NIAH-MK1 & 100.0 & 100.0 & 100.0 & 100.0 & 100.0 & 98.0 & 2.0 & N/A & 100.0 \\
NIAH-MK2 & 100.0 & 100.0 & 100.0 & 96.0 & 96.0 & 40.0 & 0.0 & N/A & 100.0 \\
NIAH-MK3 & 98.0 & 100.0 & 100.0 & 78.0 & 62.0 & 0.0 & 0.0 & N/A & 100.0 \\
NIAH-MV & 98.5 & 96.0 & 96.5 & 96.5 & 90.0 & 65.0 & 0.0 & N/A & 84.5 \\
NIAH-MQ & 99.0 & 97.5 & 97.0 & 97.5 & 96.0 & 59.0 & 0.0 & N/A & 93.0 \\
VT & 100.0 & 95.6 & 94.0 & 96.4 & 98.0 & 81.2 & 96.8 & N/A & 89.2 \\
CWE & 87.4 & 19.4 & 19.2 & 13.6 & 13.2 & 7.8 & 0.2 & N/A & 5.8 \\
FWE & 96.7 & 82.7 & 80.0 & 80.0 & 82.0 & 62.7 & 0.0 & N/A & 86.7 \\
QA-1 & 84.0 & 86.0 & 88.0 & 88.0 & 88.0 & 76.0 & 62.0 & N/A & 88.0 \\
QA-2 & 62.0 & 60.0 & 62.0 & 62.0 & 60.0 & 42.0 & 38.0 & N/A & 60.0 \\
\midrule
\textbf{Avg} & \textbf{94.3} & \textbf{87.5} & \textbf{87.4} & \textbf{85.2} & \textbf{83.5} & \textbf{57.8} & \textbf{23.2} & \textbf{N/A} & \textbf{85.2} \\
\midrule
\multicolumn{10}{l}{\emph{budget } $b=256$}\\
\midrule
NIAH-S1 & 100.0 & 100.0 & 100.0 & 100.0 & 100.0 & 100.0 & 100.0 & N/A & 100.0 \\
NIAH-S2 & 100.0 & 100.0 & 100.0 & 100.0 & 100.0 & 98.0 & 16.0 & N/A & 100.0 \\
NIAH-S3 & 100.0 & 100.0 & 100.0 & 100.0 & 100.0 & 58.0 & 0.0 & N/A & 100.0 \\
NIAH-MK1 & 100.0 & 100.0 & 100.0 & 100.0 & 100.0 & 98.0 & 16.0 & N/A & 100.0 \\
NIAH-MK2 & 100.0 & 100.0 & 100.0 & 100.0 & 100.0 & 80.0 & 46.0 & N/A & 100.0 \\
NIAH-MK3 & 98.0 & 100.0 & 100.0 & 100.0 & 90.0 & 0.0 & 0.0 & N/A & 100.0 \\
NIAH-MV & 98.5 & 98.0 & 98.0 & 97.5 & 96.5 & 88.0 & 3.0 & N/A & 97.5 \\
NIAH-MQ & 99.0 & 99.0 & 98.0 & 99.5 & 99.0 & 79.5 & 7.5 & N/A & 98.0 \\
VT & 100.0 & 99.6 & 99.6 & 99.6 & 98.8 & 84.8 & 98.0 & N/A & 95.6 \\
CWE & 87.4 & 43.4 & 46.0 & 35.2 & 34.4 & 21.6 & 0.2 & N/A & 16.4 \\
FWE & 96.7 & 90.0 & 86.7 & 88.7 & 88.7 & 75.3 & 1.3 & N/A & 84.0 \\
QA-1 & 84.0 & 88.0 & 88.0 & 88.0 & 88.0 & 86.0 & 64.0 & N/A & 84.0 \\
QA-2 & 62.0 & 60.0 & 62.0 & 62.0 & 60.0 & 52.0 & 38.0 & N/A & 62.0 \\
\midrule
\textbf{Avg} & \textbf{94.3} & \textbf{90.6} & \textbf{90.6} & \textbf{90.0} & \textbf{88.9} & \textbf{70.9} & \textbf{30.0} & \textbf{N/A} & \textbf{87.5} \\
\midrule
\multicolumn{10}{l}{\emph{budget } $b=512$}\\
\midrule
NIAH-S1 & 100.0 & 100.0 & 100.0 & 100.0 & 100.0 & 100.0 & 100.0 & 100.0 & 100.0 \\
NIAH-S2 & 100.0 & 100.0 & 100.0 & 100.0 & 100.0 & 98.0 & 94.0 & 100.0 & 100.0 \\
NIAH-S3 & 100.0 & 100.0 & 100.0 & 100.0 & 100.0 & 82.0 & 50.0 & 100.0 & 98.0 \\
NIAH-MK1 & 100.0 & 100.0 & 100.0 & 100.0 & 100.0 & 100.0 & 94.0 & 100.0 & 100.0 \\
NIAH-MK2 & 100.0 & 100.0 & 100.0 & 100.0 & 100.0 & 96.0 & 100.0 & 98.0 & 100.0 \\
NIAH-MK3 & 98.0 & 98.0 & 98.0 & 98.0 & 98.0 & 18.0 & 68.0 & 78.0 & 100.0 \\
NIAH-MV & 98.5 & 98.5 & 97.0 & 97.5 & 97.0 & 94.5 & 38.5 & 91.5 & 98.0 \\
NIAH-MQ & 99.0 & 98.5 & 98.0 & 98.5 & 98.5 & 92.0 & 65.0 & 95.0 & 98.5 \\
VT & 100.0 & 99.6 & 99.6 & 99.6 & 99.6 & 94.8 & 98.4 & 92.4 & 98.4 \\
CWE & 87.4 & 58.6 & 57.2 & 53.2 & 55.4 & 30.0 & 0.2 & 16.6 & 36.4 \\
FWE & 96.7 & 90.0 & 90.0 & 91.3 & 94.0 & 81.3 & 37.3 & 83.3 & 88.0 \\
QA-1 & 84.0 & 86.0 & 88.0 & 88.0 & 88.0 & 86.0 & 76.0 & 84.0 & 86.0 \\
QA-2 & 62.0 & 60.0 & 60.0 & 62.0 & 62.0 & 58.0 & 52.0 & 60.0 & 62.0 \\
\midrule
\textbf{Avg} & \textbf{94.3} & \textbf{91.5} & \textbf{91.4} & \textbf{91.4} & \textbf{91.7} & \textbf{79.3} & \textbf{67.2} & \textbf{84.5} & \textbf{89.6} \\
\midrule
\multicolumn{10}{l}{\emph{budget } $b=1024$}\\
\midrule
NIAH-S1 & 100.0 & 100.0 & 100.0 & 100.0 & 100.0 & 100.0 & 100.0 & 100.0 & 100.0 \\
NIAH-S2 & 100.0 & 100.0 & 100.0 & 100.0 & 100.0 & 96.0 & 100.0 & 100.0 & 100.0 \\
NIAH-S3 & 100.0 & 100.0 & 100.0 & 100.0 & 100.0 & 94.0 & 100.0 & 100.0 & 100.0 \\
NIAH-MK1 & 100.0 & 100.0 & 100.0 & 100.0 & 100.0 & 100.0 & 100.0 & 100.0 & 100.0 \\
NIAH-MK2 & 100.0 & 100.0 & 100.0 & 100.0 & 100.0 & 98.0 & 100.0 & 100.0 & 100.0 \\
NIAH-MK3 & 98.0 & 98.0 & 98.0 & 98.0 & 98.0 & 44.0 & 98.0 & 94.0 & 98.0 \\
NIAH-MV & 98.5 & 98.0 & 98.0 & 98.0 & 98.5 & 97.5 & 77.5 & 94.0 & 98.0 \\
NIAH-MQ & 99.0 & 99.5 & 99.5 & 99.5 & 98.5 & 95.5 & 92.5 & 98.0 & 99.0 \\
VT & 100.0 & 100.0 & 100.0 & 100.0 & 98.4 & 96.4 & 98.8 & 96.8 & 99.2 \\
CWE & 87.4 & 69.2 & 72.2 & 65.8 & 70.4 & 50.0 & 3.8 & 59.0 & 54.6 \\
FWE & 96.7 & 93.3 & 93.3 & 92.0 & 94.0 & 84.0 & 76.7 & 93.3 & 88.7 \\
QA-1 & 84.0 & 86.0 & 88.0 & 88.0 & 86.0 & 82.0 & 88.0 & 86.0 & 86.0 \\
QA-2 & 62.0 & 64.0 & 60.0 & 60.0 & 60.0 & 54.0 & 56.0 & 60.0 & 62.0 \\
\midrule
\textbf{Avg} & \textbf{94.3} & \textbf{92.9} & \textbf{93.0} & \textbf{92.4} & \textbf{92.6} & \textbf{84.0} & \textbf{83.9} & \textbf{90.9} & \textbf{91.2} \\
\midrule
\multicolumn{10}{l}{\emph{budget } $b=2048$}\\
\midrule
NIAH-S1 & 100.0 & 100.0 & 100.0 & 100.0 & 100.0 & 100.0 & 100.0 & 100.0 & 100.0 \\
NIAH-S2 & 100.0 & 100.0 & 100.0 & 100.0 & 100.0 & 100.0 & 100.0 & 100.0 & 100.0 \\
NIAH-S3 & 100.0 & 100.0 & 100.0 & 100.0 & 100.0 & 98.0 & 100.0 & 100.0 & 100.0 \\
NIAH-MK1 & 100.0 & 100.0 & 100.0 & 100.0 & 100.0 & 100.0 & 100.0 & 100.0 & 100.0 \\
NIAH-MK2 & 100.0 & 100.0 & 100.0 & 100.0 & 100.0 & 100.0 & 100.0 & 100.0 & 100.0 \\
NIAH-MK3 & 98.0 & 98.0 & 98.0 & 98.0 & 98.0 & 82.0 & 100.0 & 98.0 & 98.0 \\
NIAH-MV & 98.5 & 98.5 & 98.0 & 98.0 & 98.0 & 98.0 & 97.0 & 92.5 & 98.0 \\
NIAH-MQ & 99.0 & 99.5 & 98.5 & 98.5 & 98.5 & 98.0 & 98.0 & 98.0 & 99.0 \\
VT & 100.0 & 100.0 & 100.0 & 100.0 & 100.0 & 97.2 & 100.0 & 100.0 & 99.6 \\
CWE & 87.4 & 80.4 & 79.6 & 77.6 & 78.0 & 67.6 & 50.4 & 73.8 & 73.2 \\
FWE & 96.7 & 94.0 & 94.0 & 94.0 & 95.3 & 89.3 & 92.0 & 92.7 & 92.7 \\
QA-1 & 84.0 & 86.0 & 88.0 & 86.0 & 86.0 & 84.0 & 88.0 & 86.0 & 86.0 \\
QA-2 & 62.0 & 62.0 & 62.0 & 62.0 & 60.0 & 60.0 & 62.0 & 62.0 & 64.0 \\
\midrule
\textbf{Avg} & \textbf{94.3} & \textbf{93.7} & \textbf{93.7} & \textbf{93.4} & \textbf{93.4} & \textbf{90.3} & \textbf{91.3} & \textbf{92.5} & \textbf{93.1} \\
\end{longtable}
\endgroup

\subsubsection{RULER-32K}
\begingroup\scriptsize\setlength{\tabcolsep}{2.6pt}
\begin{longtable}{lrrrrrrrrr}
\caption{RULER-32K scores by task (Llama-3.1-8B), with all methods at each budget $b$.}\label{tab:appB-ruler32}\\
\toprule
Subtask & FullKV & Oracle & r8i4 & r4i4 & r2i4 & Quest & KVzip & ShadowKV & RocketKV \\
\midrule\endfirsthead
\multicolumn{10}{c}{\tablename\ \thetable\ (cont.)}\\
\toprule Subtask & FullKV & Oracle & r8i4 & r4i4 & r2i4 & Quest & KVzip & ShadowKV & RocketKV \\
\midrule\endhead
\bottomrule\endfoot
\midrule
\multicolumn{10}{l}{\emph{budget } $b=64$}\\
\midrule
NIAH-S1 & 100.0 & 100.0 & 100.0 & 100.0 & 100.0 & 100.0 & 100.0 & N/A & 100.0 \\
NIAH-S2 & 100.0 & 100.0 & 100.0 & 100.0 & 100.0 & 78.0 & 0.0 & N/A & 100.0 \\
NIAH-S3 & 100.0 & 100.0 & 100.0 & 100.0 & 92.0 & 0.0 & 0.0 & N/A & 100.0 \\
NIAH-MK1 & 100.0 & 100.0 & 98.0 & 96.0 & 94.0 & 82.0 & 0.0 & N/A & 98.0 \\
NIAH-MK2 & 100.0 & 92.0 & 86.0 & 60.0 & 64.0 & 10.0 & 0.0 & N/A & 90.0 \\
NIAH-MK3 & 100.0 & 20.0 & 20.0 & 0.0 & 0.0 & 0.0 & 0.0 & N/A & 10.0 \\
NIAH-MV & 99.5 & 80.0 & 79.5 & 77.5 & 67.0 & 22.5 & 0.0 & N/A & 26.0 \\
NIAH-MQ & 99.5 & 89.5 & 85.0 & 86.5 & 76.5 & 25.5 & 0.0 & N/A & 2.0 \\
VT & 100.0 & 80.4 & 80.4 & 75.2 & 72.0 & 66.8 & 80.0 & N/A & 11.6 \\
CWE & 28.8 & 1.0 & 1.8 & 0.8 & 0.2 & 2.6 & 0.0 & N/A & 0.4 \\
FWE & 82.7 & 58.7 & 61.3 & 56.7 & 55.3 & 38.0 & 0.0 & N/A & 56.0 \\
QA-1 & 90.0 & 88.0 & 88.0 & 88.0 & 88.0 & 48.0 & 62.0 & N/A & 86.0 \\
QA-2 & 54.0 & 50.0 & 50.0 & 50.0 & 52.0 & 28.0 & 34.0 & N/A & 44.0 \\
\midrule
\textbf{Avg} & \textbf{88.8} & \textbf{73.8} & \textbf{73.1} & \textbf{68.5} & \textbf{66.2} & \textbf{38.6} & \textbf{21.2} & \textbf{N/A} & \textbf{55.7} \\
\midrule
\multicolumn{10}{l}{\emph{budget } $b=128$}\\
\midrule
NIAH-S1 & 100.0 & 100.0 & 100.0 & 100.0 & 100.0 & 100.0 & 100.0 & N/A & 100.0 \\
NIAH-S2 & 100.0 & 100.0 & 100.0 & 100.0 & 100.0 & 92.0 & 0.0 & N/A & 100.0 \\
NIAH-S3 & 100.0 & 100.0 & 100.0 & 100.0 & 100.0 & 28.0 & 0.0 & N/A & 100.0 \\
NIAH-MK1 & 100.0 & 100.0 & 100.0 & 100.0 & 100.0 & 100.0 & 0.0 & N/A & 100.0 \\
NIAH-MK2 & 100.0 & 100.0 & 100.0 & 96.0 & 94.0 & 38.0 & 0.0 & N/A & 100.0 \\
NIAH-MK3 & 100.0 & 90.0 & 84.0 & 40.0 & 20.0 & 0.0 & 0.0 & N/A & 98.0 \\
NIAH-MV & 99.5 & 95.0 & 96.0 & 95.0 & 91.5 & 49.5 & 0.0 & N/A & 84.0 \\
NIAH-MQ & 99.5 & 96.5 & 97.5 & 98.0 & 95.0 & 53.5 & 0.0 & N/A & 95.0 \\
VT & 100.0 & 89.6 & 90.4 & 92.8 & 89.2 & 75.2 & 85.2 & N/A & 80.0 \\
CWE & 28.8 & 4.0 & 1.8 & 2.6 & 0.8 & 3.6 & 0.0 & N/A & 0.2 \\
FWE & 82.7 & 68.7 & 68.7 & 73.3 & 71.3 & 53.3 & 0.0 & N/A & 65.3 \\
QA-1 & 90.0 & 90.0 & 90.0 & 88.0 & 90.0 & 72.0 & 62.0 & N/A & 90.0 \\
QA-2 & 54.0 & 52.0 & 56.0 & 56.0 & 58.0 & 32.0 & 38.0 & N/A & 56.0 \\
\midrule
\textbf{Avg} & \textbf{88.8} & \textbf{83.5} & \textbf{83.4} & \textbf{80.1} & \textbf{77.7} & \textbf{53.6} & \textbf{21.9} & \textbf{N/A} & \textbf{82.2} \\
\midrule
\multicolumn{10}{l}{\emph{budget } $b=256$}\\
\midrule
NIAH-S1 & 100.0 & 100.0 & 100.0 & 100.0 & 100.0 & 100.0 & 100.0 & N/A & 100.0 \\
NIAH-S2 & 100.0 & 100.0 & 100.0 & 100.0 & 100.0 & 92.0 & 2.0 & N/A & 100.0 \\
NIAH-S3 & 100.0 & 100.0 & 100.0 & 100.0 & 100.0 & 44.0 & 0.0 & N/A & 100.0 \\
NIAH-MK1 & 100.0 & 100.0 & 100.0 & 100.0 & 100.0 & 100.0 & 2.0 & N/A & 100.0 \\
NIAH-MK2 & 100.0 & 100.0 & 100.0 & 100.0 & 100.0 & 58.0 & 6.0 & N/A & 100.0 \\
NIAH-MK3 & 100.0 & 100.0 & 96.0 & 92.0 & 70.0 & 0.0 & 0.0 & N/A & 100.0 \\
NIAH-MV & 99.5 & 97.5 & 98.5 & 97.5 & 97.5 & 76.5 & 0.0 & N/A & 97.0 \\
NIAH-MQ & 99.5 & 99.0 & 99.0 & 99.0 & 99.5 & 81.0 & 2.0 & N/A & 98.0 \\
VT & 100.0 & 93.6 & 92.4 & 92.4 & 91.6 & 84.4 & 91.2 & N/A & 91.6 \\
CWE & 28.8 & 5.6 & 8.4 & 3.4 & 4.0 & 4.0 & 0.0 & N/A & 1.6 \\
FWE & 82.7 & 71.3 & 71.3 & 73.3 & 72.7 & 66.0 & 0.0 & N/A & 67.3 \\
QA-1 & 90.0 & 90.0 & 90.0 & 90.0 & 90.0 & 80.0 & 64.0 & N/A & 90.0 \\
QA-2 & 54.0 & 54.0 & 56.0 & 56.0 & 56.0 & 48.0 & 36.0 & N/A & 54.0 \\
\midrule
\textbf{Avg} & \textbf{88.8} & \textbf{85.5} & \textbf{85.5} & \textbf{84.9} & \textbf{83.2} & \textbf{64.1} & \textbf{23.3} & \textbf{N/A} & \textbf{84.6} \\
\midrule
\multicolumn{10}{l}{\emph{budget } $b=512$}\\
\midrule
NIAH-S1 & 100.0 & 100.0 & 100.0 & 100.0 & 100.0 & 100.0 & 100.0 & 100.0 & 100.0 \\
NIAH-S2 & 100.0 & 100.0 & 100.0 & 100.0 & 100.0 & 98.0 & 42.0 & 100.0 & 100.0 \\
NIAH-S3 & 100.0 & 100.0 & 100.0 & 100.0 & 100.0 & 70.0 & 2.0 & 100.0 & 100.0 \\
NIAH-MK1 & 100.0 & 100.0 & 100.0 & 100.0 & 100.0 & 100.0 & 30.0 & 98.0 & 100.0 \\
NIAH-MK2 & 100.0 & 100.0 & 100.0 & 100.0 & 100.0 & 88.0 & 46.0 & 96.0 & 100.0 \\
NIAH-MK3 & 100.0 & 100.0 & 100.0 & 98.0 & 98.0 & 0.0 & 0.0 & 30.0 & 100.0 \\
NIAH-MV & 99.5 & 98.5 & 98.5 & 98.5 & 98.5 & 92.5 & 8.0 & 90.0 & 98.5 \\
NIAH-MQ & 99.5 & 97.5 & 98.0 & 98.5 & 99.5 & 90.0 & 25.0 & 95.5 & 99.5 \\
VT & 100.0 & 96.4 & 94.0 & 95.2 & 94.0 & 84.8 & 92.8 & 86.0 & 92.4 \\
CWE & 28.8 & 14.4 & 12.6 & 9.6 & 10.4 & 6.2 & 0.0 & 1.0 & 5.4 \\
FWE & 82.7 & 70.7 & 70.7 & 70.7 & 71.3 & 74.0 & 0.0 & 71.3 & 70.0 \\
QA-1 & 90.0 & 90.0 & 90.0 & 90.0 & 90.0 & 86.0 & 70.0 & 90.0 & 90.0 \\
QA-2 & 54.0 & 56.0 & 56.0 & 58.0 & 56.0 & 48.0 & 40.0 & 52.0 & 56.0 \\
\midrule
\textbf{Avg} & \textbf{88.8} & \textbf{86.4} & \textbf{86.1} & \textbf{86.0} & \textbf{86.0} & \textbf{72.1} & \textbf{35.1} & \textbf{77.7} & \textbf{85.5} \\
\midrule
\multicolumn{10}{l}{\emph{budget } $b=1024$}\\
\midrule
NIAH-S1 & 100.0 & 100.0 & 100.0 & 100.0 & 100.0 & 100.0 & 100.0 & 100.0 & 100.0 \\
NIAH-S2 & 100.0 & 100.0 & 100.0 & 100.0 & 100.0 & 98.0 & 96.0 & 100.0 & 100.0 \\
NIAH-S3 & 100.0 & 100.0 & 100.0 & 100.0 & 100.0 & 86.0 & 66.0 & 100.0 & 100.0 \\
NIAH-MK1 & 100.0 & 100.0 & 100.0 & 100.0 & 100.0 & 100.0 & 96.0 & 100.0 & 100.0 \\
NIAH-MK2 & 100.0 & 100.0 & 100.0 & 100.0 & 100.0 & 98.0 & 98.0 & 100.0 & 100.0 \\
NIAH-MK3 & 100.0 & 100.0 & 100.0 & 100.0 & 96.0 & 22.0 & 42.0 & 92.0 & 100.0 \\
NIAH-MV & 99.5 & 99.0 & 99.5 & 99.5 & 99.0 & 96.0 & 46.0 & 98.5 & 99.0 \\
NIAH-MQ & 99.5 & 98.0 & 97.5 & 97.0 & 99.0 & 94.0 & 76.5 & 98.5 & 99.0 \\
VT & 100.0 & 96.8 & 97.2 & 97.2 & 97.2 & 88.8 & 94.8 & 93.2 & 94.8 \\
CWE & 28.8 & 18.4 & 16.0 & 13.8 & 17.4 & 12.8 & 0.0 & 14.6 & 15.0 \\
FWE & 82.7 & 72.7 & 73.3 & 72.0 & 71.3 & 76.0 & 20.0 & 70.7 & 66.7 \\
QA-1 & 90.0 & 90.0 & 90.0 & 90.0 & 90.0 & 86.0 & 84.0 & 90.0 & 90.0 \\
QA-2 & 54.0 & 56.0 & 60.0 & 60.0 & 60.0 & 50.0 & 44.0 & 54.0 & 56.0 \\
\midrule
\textbf{Avg} & \textbf{88.8} & \textbf{87.0} & \textbf{87.2} & \textbf{86.9} & \textbf{86.9} & \textbf{77.5} & \textbf{66.4} & \textbf{85.5} & \textbf{86.2} \\
\midrule
\multicolumn{10}{l}{\emph{budget } $b=2048$}\\
\midrule
NIAH-S1 & 100.0 & 100.0 & 100.0 & 100.0 & 100.0 & 100.0 & 100.0 & 100.0 & 100.0 \\
NIAH-S2 & 100.0 & 100.0 & 100.0 & 100.0 & 100.0 & 96.0 & 100.0 & 100.0 & 100.0 \\
NIAH-S3 & 100.0 & 100.0 & 100.0 & 100.0 & 100.0 & 90.0 & 100.0 & 100.0 & 100.0 \\
NIAH-MK1 & 100.0 & 100.0 & 100.0 & 100.0 & 100.0 & 100.0 & 100.0 & 100.0 & 100.0 \\
NIAH-MK2 & 100.0 & 100.0 & 100.0 & 100.0 & 100.0 & 98.0 & 100.0 & 100.0 & 100.0 \\
NIAH-MK3 & 100.0 & 100.0 & 98.0 & 98.0 & 96.0 & 64.0 & 96.0 & 100.0 & 100.0 \\
NIAH-MV & 99.5 & 99.5 & 99.5 & 99.5 & 99.5 & 98.0 & 76.0 & 99.0 & 99.5 \\
NIAH-MQ & 99.5 & 99.0 & 98.5 & 99.0 & 99.0 & 97.5 & 91.0 & 99.0 & 98.0 \\
VT & 100.0 & 98.0 & 98.8 & 98.4 & 98.4 & 93.6 & 96.0 & 96.8 & 96.8 \\
CWE & 28.8 & 25.2 & 25.0 & 19.8 & 21.6 & 13.0 & 0.2 & 20.0 & 32.2 \\
FWE & 82.7 & 74.0 & 74.7 & 75.3 & 74.7 & 80.7 & 63.3 & 72.7 & 71.3 \\
QA-1 & 90.0 & 90.0 & 90.0 & 90.0 & 90.0 & 90.0 & 86.0 & 90.0 & 90.0 \\
QA-2 & 54.0 & 58.0 & 60.0 & 60.0 & 60.0 & 56.0 & 54.0 & 54.0 & 56.0 \\
\midrule
\textbf{Avg} & \textbf{88.8} & \textbf{88.0} & \textbf{88.0} & \textbf{87.7} & \textbf{87.6} & \textbf{82.8} & \textbf{81.7} & \textbf{87.0} & \textbf{88.0} \\
\end{longtable}
\endgroup

\subsubsection{InfiniteBench}
We evaluate GLM-4-9B-Chat-1M on $11$ tasks, excluding
\texttt{math\_calc} because its $30$K-token generation exceeds baseline
support under the common protocol. On this $100$K-context benchmark,
\method{} reaches FullKV parity ($43.6$ vs $43.0$ at $b{=}2048$), while
the other deployable selectors score lower.
\begingroup\scriptsize\setlength{\tabcolsep}{4.0pt}
\providecommand{\cipm}[2]{#1\,{\tiny$\pm$#2}}
\begin{longtable}{lrrrrrr}
\caption{InfiniteBench scores by task and budget (GLM-4-9B-Chat-1M). Entries are means with $95\%$ bootstrap CI
half-widths; N/A marks budgets below ShadowKV's fixed block.}\label{tab:appB-infinitebench}\\
\toprule
Task & FullKV & exact-LSE & r8i4 & RocketKV & ShadowKV & Quest \\
\midrule\endfirsthead
\multicolumn{7}{c}{\tablename\ \thetable\ (cont.)}\\
\toprule Task & FullKV & exact-LSE & r8i4 & RocketKV & ShadowKV & Quest \\
\midrule\endhead
\bottomrule\endfoot
\midrule
\multicolumn{7}{l}{\emph{budget } $b=64$}\\
\midrule
passkey & \cipm{100.0}{0.0} & \cipm{100.0}{0.0} & \cipm{100.0}{0.0} & \cipm{86.0}{9.6} & N/A & \cipm{100.0}{0.0} \\
num-str & \cipm{100.0}{0.0} & \cipm{98.0}{3.9} & \cipm{100.0}{0.0} & \cipm{68.0}{12.8} & N/A & \cipm{70.0}{12.7} \\
kv-retr & \cipm{28.0}{12.7} & \cipm{2.0}{3.9} & \cipm{0.0}{0.0} & \cipm{6.0}{6.6} & N/A & \cipm{0.0}{0.0} \\
dialog-qa & \cipm{34.0}{12.9} & \cipm{20.0}{10.9} & \cipm{16.0}{10.3} & \cipm{2.0}{3.9} & N/A & \cipm{12.0}{8.9} \\
book-sum & \cipm{27.6}{2.1} & \cipm{18.0}{2.3} & \cipm{16.4}{2.3} & \cipm{14.2}{1.8} & N/A & \cipm{14.9}{1.1} \\
book-choice & \cipm{82.0}{10.6} & \cipm{86.0}{9.5} & \cipm{82.0}{10.4} & \cipm{78.0}{11.5} & N/A & \cipm{60.0}{13.5} \\
book-qa-en & \cipm{16.2}{5.7} & \cipm{19.2}{7.2} & \cipm{18.2}{6.8} & \cipm{10.4}{5.7} & N/A & \cipm{8.1}{4.9} \\
book-qa-zh & \cipm{16.8}{3.8} & \cipm{14.7}{4.9} & \cipm{16.7}{4.9} & \cipm{13.8}{5.1} & N/A & \cipm{7.4}{2.8} \\
math-find & \cipm{30.0}{12.6} & \cipm{24.0}{11.7} & \cipm{26.0}{12.0} & \cipm{22.0}{11.3} & N/A & \cipm{28.0}{12.3} \\
code-run & \cipm{6.0}{6.5} & \cipm{0.0}{0.0} & \cipm{0.0}{0.0} & \cipm{0.0}{0.0} & N/A & \cipm{0.0}{0.0} \\
code-debug & \cipm{32.0}{13.0} & \cipm{28.0}{12.5} & \cipm{34.0}{13.2} & \cipm{36.0}{13.3} & N/A & \cipm{32.0}{12.6} \\
\midrule
\textbf{Avg} & \cipm{\textbf{43.0}}{2.1} & \cipm{\textbf{37.3}}{2.0} & \cipm{\textbf{37.2}}{2.2} & \cipm{\textbf{30.6}}{2.4} & N/A & \cipm{\textbf{30.2}}{2.6} \\
\midrule
\multicolumn{7}{l}{\emph{budget } $b=128$}\\
\midrule
passkey & \cipm{100.0}{0.0} & \cipm{100.0}{0.0} & \cipm{100.0}{0.0} & \cipm{98.0}{3.9} & N/A & \cipm{100.0}{0.0} \\
num-str & \cipm{100.0}{0.0} & \cipm{100.0}{0.0} & \cipm{100.0}{0.0} & \cipm{66.0}{13.1} & N/A & \cipm{84.0}{10.1} \\
kv-retr & \cipm{28.0}{12.7} & \cipm{2.0}{3.9} & \cipm{2.0}{3.8} & \cipm{16.0}{10.2} & N/A & \cipm{0.0}{0.0} \\
dialog-qa & \cipm{34.0}{12.9} & \cipm{36.0}{13.3} & \cipm{28.0}{12.4} & \cipm{10.0}{8.3} & N/A & \cipm{24.0}{11.8} \\
book-sum & \cipm{27.6}{2.1} & \cipm{17.8}{2.1} & \cipm{18.1}{2.3} & \cipm{15.5}{2.0} & N/A & \cipm{13.8}{1.6} \\
book-choice & \cipm{82.0}{10.6} & \cipm{82.0}{10.6} & \cipm{84.0}{10.2} & \cipm{82.0}{10.6} & N/A & \cipm{72.0}{12.5} \\
book-qa-en & \cipm{16.2}{5.7} & \cipm{18.6}{6.9} & \cipm{20.0}{6.8} & \cipm{14.3}{6.5} & N/A & \cipm{14.7}{6.9} \\
book-qa-zh & \cipm{16.8}{3.8} & \cipm{14.3}{3.9} & \cipm{14.9}{4.5} & \cipm{16.9}{6.0} & N/A & \cipm{6.3}{2.2} \\
math-find & \cipm{30.0}{12.6} & \cipm{32.0}{12.8} & \cipm{30.0}{12.6} & \cipm{26.0}{12.1} & N/A & \cipm{30.0}{12.4} \\
code-run & \cipm{6.0}{6.5} & \cipm{0.0}{0.0} & \cipm{0.0}{0.0} & \cipm{0.0}{0.0} & N/A & \cipm{2.0}{3.8} \\
code-debug & \cipm{32.0}{13.0} & \cipm{38.0}{13.5} & \cipm{34.0}{13.2} & \cipm{38.0}{13.5} & N/A & \cipm{14.0}{9.7} \\
\midrule
\textbf{Avg} & \cipm{\textbf{43.0}}{2.1} & \cipm{\textbf{40.1}}{2.4} & \cipm{\textbf{39.2}}{2.3} & \cipm{\textbf{34.8}}{2.2} & N/A & \cipm{\textbf{32.8}}{2.0} \\
\midrule
\multicolumn{7}{l}{\emph{budget } $b=256$}\\
\midrule
passkey & \cipm{100.0}{0.0} & \cipm{100.0}{0.0} & \cipm{100.0}{0.0} & \cipm{98.0}{3.9} & N/A & \cipm{100.0}{0.0} \\
num-str & \cipm{100.0}{0.0} & \cipm{100.0}{0.0} & \cipm{100.0}{0.0} & \cipm{84.0}{10.2} & N/A & \cipm{96.0}{5.4} \\
kv-retr & \cipm{28.0}{12.7} & \cipm{16.0}{10.3} & \cipm{6.0}{6.6} & \cipm{18.0}{11.0} & N/A & \cipm{0.0}{0.0} \\
dialog-qa & \cipm{34.0}{12.9} & \cipm{32.0}{12.9} & \cipm{30.0}{12.9} & \cipm{24.0}{11.9} & N/A & \cipm{16.0}{10.1} \\
book-sum & \cipm{27.6}{2.1} & \cipm{21.5}{2.4} & \cipm{21.4}{2.3} & \cipm{18.8}{1.8} & N/A & \cipm{14.8}{1.2} \\
book-choice & \cipm{82.0}{10.6} & \cipm{82.0}{10.6} & \cipm{82.0}{10.6} & \cipm{82.0}{10.6} & N/A & \cipm{76.0}{11.9} \\
book-qa-en & \cipm{16.2}{5.7} & \cipm{21.0}{7.5} & \cipm{21.6}{7.5} & \cipm{18.7}{7.3} & N/A & \cipm{17.9}{7.3} \\
book-qa-zh & \cipm{16.8}{3.8} & \cipm{16.5}{4.5} & \cipm{16.3}{4.6} & \cipm{18.3}{5.8} & N/A & \cipm{10.8}{3.6} \\
math-find & \cipm{30.0}{12.6} & \cipm{32.0}{12.9} & \cipm{30.0}{12.6} & \cipm{26.0}{12.0} & N/A & \cipm{32.0}{12.8} \\
code-run & \cipm{6.0}{6.5} & \cipm{2.0}{3.9} & \cipm{2.0}{3.8} & \cipm{0.0}{0.0} & N/A & \cipm{0.0}{0.0} \\
code-debug & \cipm{32.0}{13.0} & \cipm{38.0}{13.5} & \cipm{36.0}{13.4} & \cipm{36.0}{13.5} & N/A & \cipm{36.0}{13.4} \\
\midrule
\textbf{Avg} & \cipm{\textbf{43.0}}{2.1} & \cipm{\textbf{41.9}}{2.6} & \cipm{\textbf{40.5}}{2.4} & \cipm{\textbf{38.5}}{2.7} & N/A & \cipm{\textbf{36.3}}{2.3} \\
\midrule
\multicolumn{7}{l}{\emph{budget } $b=512$}\\
\midrule
passkey & \cipm{100.0}{0.0} & \cipm{100.0}{0.0} & \cipm{100.0}{0.0} & \cipm{96.0}{5.4} & \cipm{100.0}{0.0} & \cipm{100.0}{0.0} \\
num-str & \cipm{100.0}{0.0} & \cipm{100.0}{0.0} & \cipm{100.0}{0.0} & \cipm{68.0}{13.0} & \cipm{98.0}{3.8} & \cipm{88.0}{9.1} \\
kv-retr & \cipm{28.0}{12.7} & \cipm{24.0}{11.9} & \cipm{14.0}{9.7} & \cipm{28.0}{12.7} & \cipm{0.0}{0.0} & \cipm{0.0}{0.0} \\
dialog-qa & \cipm{34.0}{12.9} & \cipm{28.0}{12.4} & \cipm{26.0}{12.2} & \cipm{26.0}{12.2} & \cipm{14.0}{9.4} & \cipm{24.0}{11.5} \\
book-sum & \cipm{27.6}{2.1} & \cipm{24.8}{2.4} & \cipm{23.3}{2.2} & \cipm{21.6}{2.1} & \cipm{17.6}{2.0} & \cipm{12.9}{1.5} \\
book-choice & \cipm{82.0}{10.6} & \cipm{84.0}{10.1} & \cipm{82.0}{10.6} & \cipm{84.0}{10.2} & \cipm{78.0}{11.5} & \cipm{72.0}{12.4} \\
book-qa-en & \cipm{16.2}{5.7} & \cipm{20.5}{7.5} & \cipm{19.0}{7.4} & \cipm{19.5}{7.5} & \cipm{19.1}{7.0} & \cipm{14.5}{5.8} \\
book-qa-zh & \cipm{16.8}{3.8} & \cipm{18.0}{4.6} & \cipm{17.9}{4.6} & \cipm{16.7}{5.1} & \cipm{17.0}{5.2} & \cipm{10.0}{3.1} \\
math-find & \cipm{30.0}{12.6} & \cipm{32.0}{12.8} & \cipm{34.0}{13.0} & \cipm{24.0}{11.6} & \cipm{22.0}{11.4} & \cipm{32.0}{12.8} \\
code-run & \cipm{6.0}{6.5} & \cipm{2.0}{3.8} & \cipm{2.0}{3.8} & \cipm{0.0}{0.0} & \cipm{0.0}{0.0} & \cipm{0.0}{0.0} \\
code-debug & \cipm{32.0}{13.0} & \cipm{32.0}{13.1} & \cipm{34.0}{13.2} & \cipm{40.0}{13.6} & \cipm{30.0}{12.4} & \cipm{28.0}{12.5} \\
\midrule
\textbf{Avg} & \cipm{\textbf{43.0}}{2.1} & \cipm{\textbf{42.3}}{3.0} & \cipm{\textbf{41.1}}{2.4} & \cipm{\textbf{38.5}}{2.7} & \cipm{\textbf{36.0}}{2.1} & \cipm{\textbf{34.7}}{2.7} \\
\midrule
\multicolumn{7}{l}{\emph{budget } $b=1024$}\\
\midrule
passkey & \cipm{100.0}{0.0} & \cipm{100.0}{0.0} & \cipm{100.0}{0.0} & \cipm{100.0}{0.0} & \cipm{100.0}{0.0} & \cipm{100.0}{0.0} \\
num-str & \cipm{100.0}{0.0} & \cipm{100.0}{0.0} & \cipm{100.0}{0.0} & \cipm{98.0}{3.8} & \cipm{100.0}{0.0} & \cipm{82.0}{10.7} \\
kv-retr & \cipm{28.0}{12.7} & \cipm{28.0}{12.5} & \cipm{28.0}{12.5} & \cipm{36.0}{13.5} & \cipm{6.0}{6.6} & \cipm{0.0}{0.0} \\
dialog-qa & \cipm{34.0}{12.9} & \cipm{32.0}{13.0} & \cipm{36.0}{13.3} & \cipm{28.0}{12.5} & \cipm{34.0}{13.0} & \cipm{26.0}{11.9} \\
book-sum & \cipm{27.6}{2.1} & \cipm{24.9}{2.1} & \cipm{26.5}{2.1} & \cipm{22.1}{2.2} & \cipm{19.3}{2.3} & \cipm{15.6}{1.9} \\
book-choice & \cipm{82.0}{10.6} & \cipm{84.0}{10.1} & \cipm{82.0}{10.6} & \cipm{84.0}{10.2} & \cipm{78.0}{11.5} & \cipm{76.0}{11.7} \\
book-qa-en & \cipm{16.2}{5.7} & \cipm{19.0}{7.4} & \cipm{18.7}{7.2} & \cipm{19.6}{7.6} & \cipm{19.3}{7.5} & \cipm{17.2}{6.0} \\
book-qa-zh & \cipm{16.8}{3.8} & \cipm{18.0}{4.6} & \cipm{18.5}{4.8} & \cipm{16.0}{3.8} & \cipm{19.2}{4.1} & \cipm{13.7}{4.1} \\
math-find & \cipm{30.0}{12.6} & \cipm{32.0}{12.8} & \cipm{32.0}{12.8} & \cipm{30.0}{12.6} & \cipm{28.0}{12.3} & \cipm{30.0}{12.7} \\
code-run & \cipm{6.0}{6.5} & \cipm{4.0}{5.4} & \cipm{4.0}{5.4} & \cipm{2.0}{3.8} & \cipm{0.0}{0.0} & \cipm{2.0}{3.8} \\
code-debug & \cipm{32.0}{13.0} & \cipm{34.0}{13.2} & \cipm{36.0}{13.4} & \cipm{36.0}{13.4} & \cipm{32.0}{12.7} & \cipm{32.0}{12.8} \\
\midrule
\textbf{Avg} & \cipm{\textbf{43.0}}{2.1} & \cipm{\textbf{43.3}}{2.9} & \cipm{\textbf{43.8}}{2.6} & \cipm{\textbf{42.9}}{2.4} & \cipm{\textbf{39.6}}{2.0} & \cipm{\textbf{35.9}}{2.6} \\
\midrule
\multicolumn{7}{l}{\emph{budget } $b=2048$}\\
\midrule
passkey & \cipm{100.0}{0.0} & \cipm{100.0}{0.0} & \cipm{100.0}{0.0} & \cipm{98.0}{3.9} & \cipm{100.0}{0.0} & \cipm{100.0}{0.0} \\
num-str & \cipm{100.0}{0.0} & \cipm{100.0}{0.0} & \cipm{100.0}{0.0} & \cipm{92.0}{7.5} & \cipm{100.0}{0.0} & \cipm{74.0}{12.1} \\
kv-retr & \cipm{28.0}{12.7} & \cipm{30.0}{12.7} & \cipm{28.0}{12.6} & \cipm{36.0}{13.6} & \cipm{14.0}{9.7} & \cipm{0.0}{0.0} \\
dialog-qa & \cipm{34.0}{12.9} & \cipm{34.0}{13.1} & \cipm{38.0}{13.5} & \cipm{32.0}{13.0} & \cipm{32.0}{12.7} & \cipm{26.0}{12.0} \\
book-sum & \cipm{27.6}{2.1} & \cipm{27.6}{2.1} & \cipm{27.4}{2.3} & \cipm{23.7}{2.3} & \cipm{20.1}{2.0} & \cipm{16.4}{2.0} \\
book-choice & \cipm{82.0}{10.6} & \cipm{84.0}{10.1} & \cipm{82.0}{10.6} & \cipm{82.0}{10.6} & \cipm{76.0}{11.8} & \cipm{80.0}{11.1} \\
book-qa-en & \cipm{16.2}{5.7} & \cipm{17.9}{6.8} & \cipm{18.0}{6.8} & \cipm{20.5}{7.6} & \cipm{17.5}{6.8} & \cipm{16.8}{5.8} \\
book-qa-zh & \cipm{16.8}{3.8} & \cipm{17.3}{4.2} & \cipm{17.8}{4.2} & \cipm{17.0}{4.5} & \cipm{20.0}{5.4} & \cipm{16.0}{4.9} \\
math-find & \cipm{30.0}{12.6} & \cipm{32.0}{12.8} & \cipm{32.0}{12.8} & \cipm{30.0}{12.6} & \cipm{32.0}{12.8} & \cipm{32.0}{12.9} \\
code-run & \cipm{6.0}{6.5} & \cipm{6.0}{6.6} & \cipm{4.0}{5.4} & \cipm{0.0}{0.0} & \cipm{2.0}{3.8} & \cipm{2.0}{3.8} \\
code-debug & \cipm{32.0}{13.0} & \cipm{34.0}{13.2} & \cipm{32.0}{13.0} & \cipm{34.0}{13.2} & \cipm{30.0}{12.4} & \cipm{30.0}{12.5} \\
\midrule
\textbf{Avg} & \cipm{\textbf{43.0}}{2.1} & \cipm{\textbf{43.9}}{2.8} & \cipm{\textbf{43.6}}{2.6} & \cipm{\textbf{42.3}}{2.4} & \cipm{\textbf{40.3}}{2.3} & \cipm{\textbf{35.7}}{2.6} \\
\end{longtable}
\endgroup

\FloatBarrier

\subsection{Long-form reasoning}
\label{app:extres-reasoning}
\subsubsection{Baselines}
\label{app:extres-reasoning-baselines}
Reasoning scores average seeds $0$--$3$. AIME26 and MATH-500 use
$38{,}912$- and $16{,}384$-token generation caps, respectively.
KVzip and ShadowKV have no reasoning results. RocketKV is omitted:
its prompt budget $\mathrm{PB}=\mathrm{tcb}-\mathrm{max\_new}$ leaves
nearly the full prompt resident during long generation.
The reasoning baselines use these settings:
\begin{itemize}
  \setlength{\itemsep}{3pt}\setlength{\parskip}{0pt}\setlength{\topsep}{3pt}
  \item \textbf{R-KV}~\citep{cai2025rkv}. Published reasoning settings:
        $\lambda{=}0.1$, buffer $128$, retain ratio $0.2$.
  \item \textbf{TriAttention}~\citep{mao2026triattention}. The
        HF-reasoning configuration protects the prompt, selects over
        generated tokens, and reduces query-head scores by the GQA-group max.
  \item \textbf{LazyEviction}~\citep{zhang2025lazyeviction}. We use
        $\alpha{=}10^{-4}$ for Qwen3-4B and $5{\times}10^{-4}$ for Llama.
        Its published $363$-token window at $1492$ capacity cannot fit
        the smaller budgets, so we set $W\approx0.243b$ throughout the
        sweep. This scaling is our adaptation. Its per-query-head budget
        can multiply GQA-group residency by $G$.
\end{itemize}

\subsubsection{AIME26}
AIME26 combines the 2026 I and II exams ($30$ problems)~\citep{aime2026},
held after Qwen3's public release.
\begin{table}[htbp]\centering
\caption{AIME 2026 I+II accuracy across per-(layer, KV-head) token
budgets (30 problems; Qwen3-4B, thinking on, avg@$4$).}\label{tab:appB-aime26}
\small\setlength{\tabcolsep}{4.5pt}
\begin{tabular}{lrrrrrrrrr}
\toprule
$b$ & FullKV & Oracle & r8i4 & r4i4 & r2i4 & Quest & R-KV & TriAtt & LazyEvict \\
\midrule
64 & 59.2 & 0.0 & 0.0 & 0.0 & 0.8 & 0.0 & 0.0 & 0.0 & 0.0 \\
128 & 59.2 & 28.3 & 25.8 & 26.7 & 25.8 & 0.0 & 0.0 & 0.0 & 0.0 \\
256 & 59.2 & 52.5 & 51.7 & 55.0 & 54.2 & 0.0 & 0.0 & 0.0 & 0.0 \\
512 & 59.2 & 62.5 & 60.8 & 64.2 & 57.5 & 0.0 & 0.0 & 3.3 & 1.7 \\
1024 & 59.2 & 62.5 & 65.0 & 64.2 & 65.0 & 5.8 & 4.2 & 22.5 & 6.7 \\
2048 & 59.2 & 59.2 & 63.3 & 65.8 & 60.8 & 33.3 & 24.2 & 45.0 & 34.2 \\
\bottomrule
\end{tabular}
\end{table}

\subsubsection{MATH-500}
We use a $16{,}384$-token generation cap, following the reasoning-model
protocol of R-KV~\citep{cai2025rkv}.
\begin{table}[htbp]\centering
\caption{MATH-500 accuracy across per-(layer, KV-head) token budgets
(Qwen3-4B, thinking on, avg@$4$, $16{,}384$-token generation cap).}\label{tab:appB-math500}
\small\setlength{\tabcolsep}{4.5pt}
\begin{tabular}{lrrrrrrrrr}
\toprule
$b$ & FullKV & Oracle & r8i4 & r4i4 & r2i4 & Quest & R-KV & TriAtt & LazyEvict \\
\midrule
64 & 94.0 & 56.0 & 54.5 & 55.0 & 53.0 & 1.5 & 5.0 & 11.5 & 0.5 \\
128 & 94.0 & 78.0 & 80.5 & 81.5 & 82.5 & 3.0 & 8.0 & 20.5 & 1.5 \\
256 & 94.0 & 92.0 & 91.0 & 90.5 & 90.5 & 9.0 & 31.5 & 44.5 & 19.5 \\
512 & 94.0 & 94.0 & 92.0 & 92.0 & 92.5 & 43.5 & 60.0 & 66.0 & 45.0 \\
1024 & 94.0 & 94.0 & 94.0 & 91.5 & 93.5 & 68.5 & 77.0 & 81.5 & 75.0 \\
2048 & 94.0 & 93.5 & 93.5 & 93.5 & 93.5 & 87.5 & 90.5 & 91.5 & 92.0 \\
\bottomrule
\end{tabular}
\end{table}

\FloatBarrier
\subsection{GH200 decode efficiency}
\label{app:gh200-eff}
\subsubsection{GPU-resident batch-1 decode}
\label{app:gh200-eff-resident}
Figure~\ref{fig:eff-glm-dense}(a--c) uses one GH200 with
GLM-4-9B-Chat-1M, TP$=1$, bf16 weights and GPU-resident original KV,
$B{=}16$, full CUDA-graph decode, $8192$-token prefill chunks, and no
prefix cache. Runs use greedy generation with EOS ignored. Each
$16$K--$512$K short run emits $100$ tokens. TTFT includes \method{}
summary or Quest envelope construction; mean ITL averages the $99$
subsequent \texttt{engine.step()} calls.
Table~\ref{tab:appB-gh200-sweep} gives the values behind panels (a,b).

Quest is our optimized GPU port. It selects up to $2048$ tokens per
query head and keeps two of its $40$ layers dense. Quest-U adapts the
authors' upstream CUDA pipeline to bf16/GQA and vLLM, retaining the
same per-query-head budget and dense layers.

\begin{table}[!htbp]
\centering
\small
\setlength{\tabcolsep}{2.5pt}
\caption{GH200 batch-$1$ context sweep with GPU-resident bf16 KV.
Quest is the optimized port plotted in Fig.~\ref{fig:eff-glm-dense};
Quest-U adapts the upstream CUDA pipeline. Build-inclusive TTFT averages
three prefills except Quest-U at $512$K (one); mean ITL averages $99$
complete decode steps.}
\label{tab:appB-gh200-sweep}
\label{tab:appB-eff-full}
\begin{tabular}{l rrrrr rrrrr}
\toprule
& \multicolumn{5}{c}{TTFT (s)} & \multicolumn{5}{c}{Mean ITL (ms)} \\
\cmidrule(lr){2-6}\cmidrule(lr){7-11}
Context & FA3 & FI & Quest & Quest-U & r4i4 & FA3 & FI & Quest & Quest-U & r4i4 \\
\midrule
$16$K  & $0.61$ & $0.63$ & $0.61$ & $0.62$ & $0.65$ & $7.08$ & $7.07$ & $7.54$ & $7.85$ & $7.05$ \\
$32$K  & $1.52$ & $1.60$ & $1.53$ & $1.53$ & $1.60$ & $7.49$ & $7.43$ & $7.69$ & $8.66$ & $7.20$ \\
$64$K  & $4.28$ & $4.53$ & $4.28$ & $4.28$ & $4.41$ & $8.21$ & $8.14$ & $7.82$ & $10.23$ & $7.55$ \\
$128$K & $13.47$ & $14.47$ & $13.49$ & $13.50$ & $13.82$ & $9.68$ & $9.56$ & $8.11$ & $13.38$ & $8.26$ \\
$256$K & $47.47$ & $50.44$ & $47.53$ & $47.54$ & $48.22$ & $12.59$ & $12.42$ & $9.13$ & $19.65$ & $9.01$ \\
$512$K & $178.62$ & $186.53$ & $178.53$ & $178.45$ & $179.96$ & $18.42$ & $18.03$ & $10.74$ & $32.17$ & $10.11$ \\
\bottomrule
\end{tabular}
\end{table}

Panel (c) follows one continuous $32$K-prefill, $128$K-generation
request per method with resident KV. The plotted ITL averages \texttt{engine.step()} time over $10$ tokens;
the final $11$-step window includes request completion.
Table~\ref{tab:appB-gh200-long} includes all post-first-token intervals.

\begin{table}[!htbp]
\centering
\small
\setlength{\tabcolsep}{9pt}
\caption{GH200 long generation: $32$K input and $128$K new tokens with
GPU-resident KV. All $131{,}071$ decode intervals are included.}
\label{tab:appB-gh200-long}
\begin{tabular}{l rrr}
\toprule
Method & TTFT (s) & Mean ITL (ms) & Decode wall (s) \\
\midrule
FA3 & $1.52$ & $8.89$ & $1167.4$ \\
FlashInfer & $1.59$ & $8.81$ & $1157.2$ \\
Quest & $1.53$ & $8.04$ & $1056.9$ \\
\method{} r4i4 & $1.62$ & $8.10$ & $1064.5$ \\
\bottomrule
\end{tabular}
\end{table}

Figure~\ref{fig:eff-glm-dense}(d) pairs the r4i4, Quest, and FullKV
AIME26 avg@$4$ scores in Table~\ref{tab:appB-aime26} with separate Qwen3-4B latency
runs. Each timing run has a $100$-token input and $38{,}912$ generated
tokens; mean ITL averages the $38{,}911$ complete decode steps.

\FloatBarrier
\subsubsection{Host-backed batched decode}
\label{app:gh200-eff-batched}
Table~\ref{tab:gh200-batched} uses the same GH200 and model. Dense KV
stays in GPU HBM. \method{} r4i4 keeps full KV in Grace memory and
frequently used pages in a smaller GPU cache, leaving room for more
simultaneous requests. Each row reports throughput at its shown batch,
so the speedup measures serving capacity even when \method{} takes
longer per decode step. Timings cover $192$ decode steps without cache
logging. Hot-page shares come from separate runs that generated the same
tokens, so counting cache accesses did not slow the reported throughput.

\clearpage
\FloatBarrier
\subsection{KV-cache quantization (KIVI): full grid}
\label{app:extres-kivi}

Table~\ref{tab:appB-kivi} gives the full KIVI~\citep{liu2024kivi}
grid: $4$- and $2$-bit KV, both summary-building orders, all budgets
and ranks on RULER-32K ($13$ tasks).

\begin{table}[H]\centering
\caption{RULER-32K average across KV-cache precision, summary rank, and
budget (13 tasks, Llama-3.1-8B). \textsc{post} builds summaries from bf16
keys; \textsc{pre} builds them from quantized keys.}
\label{tab:appB-kivi}
\footnotesize
\setlength{\tabcolsep}{4.5pt}
\begin{tabular}{r l r rr rr}
\toprule
& & & \multicolumn{2}{c}{4-bit} & \multicolumn{2}{c}{2-bit}\\
\cmidrule(lr){4-5}\cmidrule(lr){6-7}
$b$ & rank & bf16 & \textsc{post} & \textsc{pre} & \textsc{post} & \textsc{pre}\\
\midrule
$64$   & $r{=}8$ & $73.1$ & $72.8$ & $73.0$ & $68.0$ & $67.6$\\
       & $r{=}4$ & $68.5$ & $68.6$ & $68.8$ & $64.9$ & $63.7$\\
       & $r{=}2$ & $66.2$ & $65.4$ & $65.9$ & $60.0$ & $60.5$\\
\midrule
$128$  & $r{=}8$ & $83.4$ & $83.3$ & $83.7$ & $78.7$ & $76.7$\\
       & $r{=}4$ & $80.1$ & $78.1$ & $79.4$ & $75.6$ & $74.8$\\
       & $r{=}2$ & $77.7$ & $76.1$ & $76.5$ & $72.7$ & $72.8$\\
\midrule
$256$  & $r{=}8$ & $85.5$ & $85.2$ & $85.2$ & $81.9$ & $80.7$\\
       & $r{=}4$ & $84.9$ & $83.8$ & $84.4$ & $80.4$ & $79.0$\\
       & $r{=}2$ & $83.2$ & $82.7$ & $82.2$ & $78.2$ & $76.5$\\
\midrule
$512$  & $r{=}8$ & $86.1$ & $86.2$ & $85.9$ & $83.4$ & $82.8$\\
       & $r{=}4$ & $86.0$ & $85.7$ & $85.7$ & $83.7$ & $82.7$\\
       & $r{=}2$ & $86.0$ & $85.6$ & $85.3$ & $81.6$ & $80.7$\\
\midrule
$1024$ & $r{=}8$ & $87.2$ & $87.3$ & $87.1$ & $84.9$ & $85.3$\\
       & $r{=}4$ & $86.9$ & $86.6$ & $86.5$ & $83.8$ & $84.3$\\
       & $r{=}2$ & $86.9$ & $86.6$ & $86.5$ & $84.0$ & $83.2$\\
\midrule
$2048$ & $r{=}8$ & $88.0$ & $87.6$ & $87.4$ & $85.0$ & $85.3$\\
       & $r{=}4$ & $87.7$ & $87.2$ & $87.6$ & $84.5$ & $84.9$\\
       & $r{=}2$ & $87.6$ & $87.5$ & $87.8$ & $84.8$ & $84.3$\\
\bottomrule
\end{tabular}
\end{table}

\section{Extended Studies}
\label{app:extstudies}
\FloatBarrier
Unless noted, experimental settings follow App.~\ref{app:extres}.

\subsection{GQA}
\label{app:extstudies-heads}
With exact per-head scores, share-average leads the other shared-page
combine rules on all measures in Table~\ref{tab:appC-combine}.

\begin{table}[!htb]\centering
\caption{Per-query-head coverage under GQA combine rules at $b{=}2048$
using exact per-head scores (Qwen3-4B, $G{=}4$). Mean reports
group-average coverage; worst head uses the minimum across a group's
query heads. Both summarize evaluated traces; bold marks \method{}'s rule.}
\label{tab:appC-combine}
\small
\setlength{\tabcolsep}{5pt}
\begin{tabular}{l cc cc}
\toprule
& \multicolumn{2}{c}{reasoning} & \multicolumn{2}{c}{retrieval} \\
\cmidrule(lr){2-3}\cmidrule(lr){4-5}
Combine rule & mean & worst head & mean & worst head \\
\midrule
\textbf{Share-average} & $\mathbf{0.872}$ & $\mathbf{0.785}$ & $\mathbf{0.879}$ & $\mathbf{0.787}$ \\
Group-max & $0.850$ & $0.748$ & $0.860$ & $0.757$ \\
Group-mass & $0.856$ & $0.756$ & $0.866$ & $0.763$ \\
\midrule
Per-head oracle (up to $G{\times}$ traffic/work) & $0.893$ & $0.812$ & $0.898$ & $0.812$ \\
FullKV & $1.000$ & $1.000$ & $1.000$ & $1.000$ \\
\bottomrule
\end{tabular}
\end{table}

\subsection{Rotary Position Embedding}
\label{app:extstudies-rope}
Post-RoPE fits, used by \method{}, give similar mass and recall to
pre-RoPE fits (Table~\ref{tab:appC-rope}).

\begin{table}[!htb]\centering
\caption{RoPE placement for page-local rank-$8$ \method{} summaries on
RULER-32K ($B{=}16$, $b{=}1024$). Entries report retained mass and
exact-LSE page recall.}
\label{tab:appC-rope}
\footnotesize
\setlength{\tabcolsep}{4.5pt}
\begin{tabular}{@{}l rr rr@{}}
\toprule
& \multicolumn{2}{c}{Llama-3.1-8B} & \multicolumn{2}{c}{Qwen3-4B}\\
\cmidrule(lr){2-3}\cmidrule(lr){4-5}
RoPE fit & mass & recall & mass & recall\\
\midrule
pre-RoPE  & 0.8936 & 0.9461 & 0.8770 & 0.9362\\
post-RoPE & 0.8935 & 0.9365 & 0.8767 & 0.9306\\
\bottomrule
\end{tabular}
\end{table}

\subsection{Page size}
\label{app:extstudies-pagesize}
At $B{=}16$, rank $8$ captures a median $89\%$ of within-page key
variance on both models, versus $67.8\%$ for size-matched isotropic
pages. Table~\ref{tab:appC-pagesize}
extends Table~\ref{tab:pagesize-main} across budgets and page sizes.

\begin{table}[H]\centering
\caption{LongBench-v1 and RULER accuracy across token budget and page
size (Llama-3.1-8B, rank-$8$ int4 basis). FullKV is independent of
budget and page size.}
\label{tab:appC-pagesize}
\footnotesize
\setlength{\tabcolsep}{4.5pt}
\begin{tabular}{r rrr rrr rrr}
\toprule
& \multicolumn{3}{c}{LongBench-v1}
& \multicolumn{3}{c}{RULER-16K}
& \multicolumn{3}{c}{RULER-32K}\\
\cmidrule(lr){2-4}\cmidrule(lr){5-7}\cmidrule(lr){8-10}
$b$ & $B{=}16$ & $B{=}32$ & $B{=}64$
    & $B{=}16$ & $B{=}32$ & $B{=}64$
    & $B{=}16$ & $B{=}32$ & $B{=}64$\\
\midrule
FullKV & \multicolumn{3}{c}{$47.0$} & \multicolumn{3}{c}{$94.3$} & \multicolumn{3}{c}{$88.8$}\\
\midrule
$256$  & $46.7$ & $46.8$ & $46.8$ & $90.6$ & $89.7$ & $89.0$ & $85.5$ & $82.7$ & $83.2$\\
$512$  & $47.1$ & $46.9$ & $47.0$ & $91.4$ & $91.3$ & $91.5$ & $86.1$ & $83.8$ & $83.7$\\
$1024$ & $47.2$ & $47.1$ & $47.1$ & $93.0$ & $92.1$ & $92.7$ & $87.2$ & $83.5$ & $84.8$\\
$2048$ & $47.2$ & $47.1$ & $47.0$ & $93.7$ & $93.0$ & $93.6$ & $88.0$ & $84.5$ & $85.1$\\
\bottomrule
\end{tabular}
\end{table}

\subsection{Mixture-of-experts: ERNIE-4.5-21B-A3B-Thinking}
\label{app:extstudies-moe}
Table~\ref{tab:appC-ernie-reasoning} extends the reasoning study to
ERNIE-4.5-21B-A3B-Thinking~\citep{ernie45report}, a mixture-of-experts
model with $21$B total and ${\sim}3$B active parameters.

\begin{table}[!htb]\centering
\caption{Accuracy on AIME26 and the first $50$ MATH-500 questions across
token budgets for ERNIE-4.5-21B-A3B-Thinking. Seed $0$; exact-LSE and
\method{} ranks $8$, $4$, and $2$ with int4 bases.}
\label{tab:appC-ernie-reasoning}
\footnotesize
\setlength{\tabcolsep}{4.5pt}
\begin{minipage}[t]{0.485\linewidth}\centering
\textbf{AIME26}\quad(FullKV: $66.7$)\\[2.5pt]
\begin{tabular}{r r rrr}
\toprule
& & \multicolumn{3}{c}{\method{}}\\
\cmidrule(lr){3-5}
$b$ & Oracle & r8i4 & r4i4 & r2i4\\
\midrule
$64$   & $0.0$  & $0.0$  & $0.0$  & $0.0$\\
$128$  & $33.3$ & $23.3$ & $36.7$ & $30.0$\\
$256$  & $46.7$ & $56.7$ & $50.0$ & $50.0$\\
$512$  & $63.3$ & $63.3$ & $56.7$ & $56.7$\\
$1024$ & $60.0$ & $60.0$ & $66.7$ & $63.3$\\
$2048$ & $60.0$ & $66.7$ & $66.7$ & $66.7$\\
\bottomrule
\end{tabular}
\end{minipage}\hfill
\begin{minipage}[t]{0.485\linewidth}\centering
\textbf{MATH-500}\quad(FullKV: $90.0$)\\[2.5pt]
\begin{tabular}{r r rrr}
\toprule
& & \multicolumn{3}{c}{\method{}}\\
\cmidrule(lr){3-5}
$b$ & Oracle & r8i4 & r4i4 & r2i4\\
\midrule
$64$   & $46.0$ & $48.0$ & $54.0$ & $48.0$\\
$128$  & $78.0$ & $72.0$ & $76.0$ & $76.0$\\
$256$  & $82.0$ & $88.0$ & $90.0$ & $88.0$\\
$512$  & $90.0$ & $88.0$ & $90.0$ & $90.0$\\
$1024$ & $92.0$ & $90.0$ & $88.0$ & $88.0$\\
$2048$ & $86.0$ & $88.0$ & $92.0$ & $86.0$\\
\bottomrule
\end{tabular}
\end{minipage}
\end{table}

\FloatBarrier
\subsection{Large dense: DeepSeek-R1-Distill-Qwen-32B}
\label{app:extstudies-r1d}
Table~\ref{tab:appC-r1d-reasoning} repeats the reasoning study on the
$32$B dense DeepSeek-R1-Distill-Qwen-32B~\citep{deepseekr1}, using the
main $B{=}16$ page size.

\begin{table}[!htb]\centering
\caption{Accuracy on AIME26 and the first $50$ MATH-500 questions across
token budgets for DeepSeek-R1-Distill-Qwen-32B. Seed $0$; exact-LSE and
\method{} ranks $8$, $4$, and $2$ with int4 bases.}
\label{tab:appC-r1d-reasoning}
\footnotesize
\setlength{\tabcolsep}{4.5pt}
\begin{minipage}[t]{0.485\linewidth}\centering
\textbf{AIME26}\quad(FullKV: $53.3$)\\[2.5pt]
\begin{tabular}{r r rrr}
\toprule
& & \multicolumn{3}{c}{\method{}}\\
\cmidrule(lr){3-5}
$b$ & Oracle & r8i4 & r4i4 & r2i4\\
\midrule
$64$   & $6.7$  & $16.7$ & $23.3$ & $16.7$\\
$128$  & $50.0$ & $53.3$ & $50.0$ & $43.3$\\
$256$  & $46.7$ & $56.7$ & $50.0$ & $56.7$\\
$512$  & $70.0$ & $46.7$ & $56.7$ & $66.7$\\
$1024$ & $60.0$ & $60.0$ & $60.0$ & $53.3$\\
$2048$ & $63.3$ & $60.0$ & $63.3$ & $56.7$\\
\bottomrule
\end{tabular}
\end{minipage}\hfill
\begin{minipage}[t]{0.485\linewidth}\centering
\textbf{MATH-500}\quad(FullKV: $94.0$)\\[2.5pt]
\begin{tabular}{r r rrr}
\toprule
& & \multicolumn{3}{c}{\method{}}\\
\cmidrule(lr){3-5}
$b$ & Oracle & r8i4 & r4i4 & r2i4\\
\midrule
$64$   & $76.0$ & $82.0$ & $80.0$ & $76.0$\\
$128$  & $94.0$ & $90.0$ & $88.0$ & $92.0$\\
$256$  & $94.0$ & $92.0$ & $94.0$ & $94.0$\\
$512$  & $92.0$ & $94.0$ & $90.0$ & $92.0$\\
$1024$ & $94.0$ & $94.0$ & $94.0$ & $94.0$\\
$2048$ & $90.0$ & $94.0$ & $94.0$ & $94.0$\\
\bottomrule
\end{tabular}
\end{minipage}
\end{table}

\FloatBarrier
\subsection{Learned sparse attention: Qwen3.8-Flash-Next}
\label{app:extstudies-qsa}
In Qwen3.8-Flash-Next~\citep{qwen38next}, we replace decode-time
selection in its 12 learned-QSA layers; prefill and Gated DeltaNet layers
remain unchanged. Methods share the FP8 checkpoint, bf16 KV cache,
H200 TP2 setup, prompts, and sampling seeds.

\begin{table}[H]
\centering
\caption{Qwen3.8-Flash-Next-FP8 accuracy across token budgets on AIME26
and the first $50$ MATH-500 questions (one sample per question).}
\label{tab:appC-qsa}
\footnotesize
\setlength{\tabcolsep}{4.5pt}
\begin{tabular}{@{}lrrrrrrrr@{}}
\toprule
& \multicolumn{4}{c}{AIME26} & \multicolumn{4}{c}{MATH-500} \\
\cmidrule(lr){2-5}\cmidrule(l){6-9}
Method & $b{=}64$ & $128$ & $256$ & $512$
       & $b{=}64$ & $128$ & $256$ & $512$ \\
\midrule
QSA (native)       & $66.7$ & $76.7$ & $80.0$ & $86.7$ & $100.0$ & $98.0$ & $94.0$ & $96.0$ \\
\method{} r1i4    & $53.3$ & $80.0$ & $86.7$ & $86.7$ &  $96.0$ & $96.0$ & $96.0$ & $96.0$ \\
Quest              & $20.0$ & $46.7$ & $80.0$ & $76.7$ &  $84.0$ & $96.0$ & $96.0$ & $96.0$ \\
\bottomrule
\end{tabular}
\end{table}

QSA may append three tail tokens, while \method{} ($B{=}4$)
and Quest count every selected token within $b$.

\FloatBarrier
\subsection{H200 decode efficiency: measurement protocol and full numbers}
\label{app:h200-eff-protocol}
\label{app:h200-eff-config}
For this GPU-resident byte model, a bf16 KV page at $B{=}16$ and
$d{=}128$ holds $8192$ B. Including scales, rank-$8$, $4$, and $2$
summaries use $818$, $490$, and $326$ B/page, giving summary fractions
$\alpha=0.100$, $0.060$, and $0.040$. Modeled per-step HBM reads are
$\alpha N+b$ KV-token equivalents.

\textbf{Configuration.} We use GLM-4-9B-Chat-1M on one H200 NVL
(TP$=1$, bf16, $B{=}16$, $b{=}2048$). Latency comparisons are among
vLLM decode backends: \method{} r8i4 and the dense FlashAttention-3
(FA3) and FlashInfer controls. Only runs with active r8i4
scoring on every decode step are included; prefill uses FA3.

\textbf{Measurement.} After warm-up, clock stabilization, and CUDA-graph
capture, one prefill feeds three $64$-step decode windows. We retain
lockstep windows without NVML-detected throttling, excluding
end-of-sequence effects. TTFT uses three repeated prefills; its
build-inclusive mode adds the one-time summary build. Table~\ref{tab:appC-eff-full}
reports TPOT mean\,$\pm$\,sd across $192$ steps and build
mean\,$\pm$\,sd across $8$--$10$ boots.

\begin{table}[H]
\centering
\small
\setlength{\tabcolsep}{4.5pt}
\caption{Decode efficiency from $16$K to $1$M context at batch $1$, TP$=1$
(GLM-4-9B-Chat-1M, one H200 NVL). TPOT reports mean\,$\pm$\,sd over $192$
steps; build reports mean\,$\pm$\,sd over $8$--$10$ boots.}
\label{tab:appC-eff-full}
\begin{tabular}{l ccc c cc c}
\toprule
& \multicolumn{3}{c}{TPOT (ms)} & & \multicolumn{2}{c}{steady TTFT (s)} & build \\
\cmidrule(lr){2-4}\cmidrule(lr){6-7}
ctx & FA3 & FI & r8i4 & speedup & FA3 & r8i4 & (ms) \\
\midrule
16K  & $6.93\,\pm\,0.05$ & $7.20\,\pm\,0.05$ & $\mathbf{7.08}\,\pm\,0.21$ & $0.98\times$ & $0.63$ & $0.63$ & $547\,\pm\,349$ \\
32K  & $7.33\,\pm\,0.03$ & $6.91\,\pm\,0.04$ & $\mathbf{7.24}\,\pm\,0.21$ & $0.96\times$ & $1.56$ & $1.57$ & $810\,\pm\,525$ \\
64K  & $7.97\,\pm\,0.08$ & $8.06\,\pm\,0.04$ & $\mathbf{7.38}\,\pm\,0.20$ & $1.08\times$ & $4.42$ & $4.41$ & $916\,\pm\,395$ \\
128K & $8.98\,\pm\,0.05$ & $8.74\,\pm\,0.05$ & $\mathbf{8.27}\,\pm\,0.22$ & $1.06\times$ & $13.9$ & $14.1$ & $2904\,\pm\,2282$ \\
256K & $11.63\,\pm\,0.09$ & $11.89\,\pm\,0.08$ & $\mathbf{8.96}\,\pm\,0.23$ & $1.30\times$ & $49.6$ & $49.7$ & $1323\,\pm\,317$ \\
512K & $17.24\,\pm\,0.18$ & $16.17\,\pm\,0.18$ & $\mathbf{10.47}\,\pm\,0.28$ & $1.54\times$ & $184.1$ & $184.2$ & $1993\,\pm\,253$ \\
1M   & $27.50\,\pm\,0.40$ & $26.19\,\pm\,0.45$ & $\mathbf{12.95}\,\pm\,0.51$ & $2.02\times$ & $699.0$ & $699.4$ & $3717\,\pm\,756$ \\
\bottomrule
\end{tabular}
\end{table}

\textbf{Batched decode (Table~\ref{tab:appC-batched}).} For each
backend and context, we sweep batch size on one live batch without
re-prefill, applying the same window checks. Speedup is against the
faster dense backend; batch-$1$ values reuse Table~\ref{tab:appC-eff-full}.
The frontier is the largest batch fitting one H200 NVL at $90\%$
utilization; larger combinations are marked OOM.

\begin{table}[H]
\centering
\small
\setlength{\tabcolsep}{4pt}
\caption{Batched decode with GLM-4-9B-Chat-1M on one H200 NVL (r8i4,
$b{=}2048$). Cells give \method{} per-token speedup over the faster of
FA3 and FlashInfer at matched context and batch, measured in three
$64$-step windows on one live batch. Bold marks the largest batch
feasible for both arms at $90\%$ GPU memory utilization; the final
column gives \method{} aggregate throughput there.}
\label{tab:appC-batched}
\begin{tabular*}{\linewidth}{@{\extracolsep{\fill}} l cccccc c@{}}
\toprule
 & \multicolumn{6}{c}{batch size} & frontier \\
\cmidrule(lr){2-7}
context & $1$ & $2$ & $4$ & $8$ & $16$ & $32$ & \method{} (tok/s) \\
\midrule
$256$K & $1.30$ & $1.43$ & $\mathbf{1.80}$ & OOM & OOM & OOM & $275$ \\
$128$K & $1.06$ & $1.21$ & $1.44$ & $\mathbf{1.59}$ & OOM & OOM & $483$ \\
$64$K  & $1.08$ & $1.08$ & $1.17$ & $1.33$ & $\mathbf{1.43}$ & OOM & $872$ \\
$32$K  & $0.96$ & $1.01$ & $1.02$ & $1.07$ & $1.14$ & $\mathbf{1.26}$ & $1503$ \\
\bottomrule
\end{tabular*}
\end{table}

\end{document}